\documentclass[twoside,11pt]{article}

\usepackage{jmlr2e}
\usepackage{amsmath}
\usepackage{dsfont}
\usepackage{xcolor}
\usepackage{algorithm}
\usepackage{algpseudocode}
\usepackage{stmaryrd} %
\usepackage{subcaption}
\usepackage{hyperref}
\usepackage{csquotes}
\usepackage{comment}

\DeclareMathOperator*{\argmin}{argmin}

\newcommand{\KL}{\operatorname{KL}}

\providecommand{\diff}{\mathrm{d}}

\providecommand{\oh}{o}
\providecommand{\Oh}{\mathcal{O}}
\providecommand{\eps}{\varepsilon}
\providecommand{\MC}{\mathcal{M}}
\providecommand{\NC}{\mathcal{N}}
\providecommand{\XC}{\mathcal{X}}

\providecommand{\PC}{\mathcal{P}}
\providecommand{\KCo}{\mathcal{K}^\circ}
\providecommand{\KC}{\mathcal{K}}
\providecommand{\vol}{\operatorname{vol}}
\providecommand{\reals}{\mathbb{R}}
\providecommand{\eps}{\varepsilon}
\providecommand{\expec}{\mathbb{E}}
\providecommand{\1}{\mathbf{1}}
\providecommand{\ones}{\mathbf{1}}
\providecommand{\0}{\mathbf{0}}
\providecommand{\tr}{\operatorname{tr}}
\providecommand{\var}{\operatorname{Var}}
\providecommand{\vect}{\operatorname{vec}}
\providecommand{\diag}{\operatorname{diag}}
\providecommand{\argmin}{\operatorname{argmin}}
\providecommand{\vM}{\vol(\MC)}
\providecommand{\hol}{\operatorname{Ho}}

\usepackage{caption}
\ShortHeadings{Manifold Learning with Sparse OT}{Zhang, Mordant,  Matsumoto and Schiebinger}
\firstpageno{1}

\begin{document}

\title{Manifold Learning with Sparse Regularised Optimal Transport}

\author{\name Stephen Zhang* \email stephenz@student.unimelb.edu.au \\
       \addr School of Mathematics and Statistics\\
       University of Melbourne\\
       Melbourne, Australia
       \AND Gilles Mordant* \email gilles.mordant@uni-goettingen.de \\
       \addr Institut f\"ur Mathematische Stochastik \\
       Universit\"at G\"ottingen \\
       G\"ottingen, Germany
       \AND Tetsuya Matsumoto \\
       \addr Department of Mathematics\\
       University of British Columbia\\
       Vancouver, Canada
       \AND Geoffrey Schiebinger \email geoff@math.ubc.ca \\
       \addr Department of Mathematics\\
       University of British Columbia\\
       Vancouver, Canada \\
       \vspace{0.25em} \\ 
       * Both authors contributed equally.
       }

\editor{}

\maketitle

\begin{abstract}%
    Manifold learning is a central task in modern statistics and data science. Many datasets (cells, documents, images, molecules) can be represented as point clouds embedded in a high dimensional ambient space, however the degrees of freedom intrinsic to the data are usually far fewer than the number of ambient dimensions. The task of detecting a latent manifold along which the data are embedded is a prerequisite for a wide family of downstream analyses. Real-world datasets are subject to noisy observations and sampling, so that distilling information about the underlying manifold is a major challenge. We propose a method for manifold learning that utilises a symmetric version of optimal transport with a quadratic regularisation that constructs a \emph{sparse} and \emph{adaptive} affinity matrix, that can be interpreted as a generalisation of the bistochastic kernel normalisation. 
    We prove that the resulting kernel is consistent with a Laplace-type operator in the continuous limit, establish robustness to heteroskedastic noise and exhibit these results in numerical experiments. We identify a highly efficient computational scheme for computing this optimal transport for discrete data and demonstrate that it outperforms competing methods in a set of examples. 
\end{abstract}


\section{Introduction}

Real-world datasets are increasingly high-dimensional with imaging and biological measurements frequently being made in tens of thousands of dimensions. Typically however the degrees of freedom intrinsic to the underlying process being observed are far fewer than the number of observed dimensions, and the data points concentrate on a latent lower dimensional subset of the ambient high-dimensional space.
To provide a concrete example, each data point in the MNIST dataset has an intrinsic dimensionality of around 15 \citep{hein2005intrinsic}, even though the pictures are $28\times 28 = 784$ pixels large.

\bgroup
Real world applications in general deviate from the ideal of data living on a smooth subspace. 
Owing to measurement noise or stochastic effects underlying the true data generation mechanism, a more natural, down-to-earth model is to assume that the data lives on a manifold and is then corrupted by noise; the latter not necessarily being homogeneous. 
\egroup
The first step towards tackling the challenge of extracting useful signal from noisy high dimensional datasets is to identify this latent lower dimensional structure. To do so, analysis workflows almost always rely on affinity matrices or neighbourhood graphs that capture the local geometry, as we will develop upon in Section~\ref{sec: ManiGraph}.

In recent years, the utility of optimal transport and related ideas to applications in statistics and data science has been acknowledged by many groups for a plethora of purposes: biology \citep{schiebinger2019optimal,zhang2021optimal,lavenant2021towards}, statistics \citep{
hallin2021multivariate,hallin2021distribution,nies2021transport,zhang2021unified,hundrieser2022empirical,mordant2022measuring,hundrieser2023empirical}, machine learning \citep{cuturi2013sinkhorn,gulrajani2017improved, feydy2019interpolating}, and economics \citep{carlier2016vector,torous2021optimal,hallin2022center}, to name a few. This list is by no means exhaustive and we encourage the interested reader to consult the references therein.

We show that manifold learning is one further domain where we draw on inspiration from optimal transport. We provide the required background on optimal transport for this particular setting and summarise the main contributions of this paper in Section~\ref{sec: MainContrib}. Section~\ref{sec: Theory} then contains our theoretical results and investigations, Section~\ref{sec: Perf} presents further use cases and simulations. We conclude by stating a few open questions for further study in Section~\ref{sec: Conclusion}. Additional elements and proofs are provided in the appendices.

\subsection{Manifold learning: graphs, affinity matrices, and operators}
\label{sec: ManiGraph}

The setting we consider in this paper is a $d$-dimensional, compact and smooth Riemannian manifold $(\MC,\rm g)$ isometrically embedded in $\reals^p$ via the embedding $\iota : \MC \hookrightarrow \reals^p$, with the ambient dimension $p$ being much larger than the intrinsic dimension $d$. This is the standard scenario in which the task of manifold learning arises \citep{meila2024manifold, coifman2006diffusion, belkin2003laplacian}, although we remark that elements of our approach are relevant to more general scenarios and we explore some of these in Remark \ref{rmk:rkhs}. In what follows, we denote by $\iota_*$, the differential of this embedding.  Our observations are a set of points sampled from the manifold, $\XC = \{ x_i \}_{i=1}^N \subset \reals^p$ where each $x_i$ is produced by embedding a point $y_i \in \MC$ into $\reals^p$, possibly with noise. These can be random or deterministic in nature. In the random case, one may consider the $y_i$ to be sampled according to some density $p$ supported on $\MC$. Then, let 
\[
C_{ij} = C(x_i, x_j) = \frac{1}{2} \| x_i - x_j \|_2^2
\] 
be the matrix of pairwise distances measured in the ambient Euclidean space $\reals^p$. Let $\mu = N^{-1} \sum_{i=1}^N \delta_{x_i}$ be the empirical measure constructed from the observed point set. 

\subsubsection{Affinity matrices} Starting from pairwise distances in the ambient Euclidean space, the characteristic workflow for a majority of manifold learning approaches \citep{coifman2005geometric, saul2000introduction, tenenbaum2000global, zhang2004principal} is to construct from $C_{ij}$ an \emph{affinity matrix} $W_{ij} \ge 0$ which can be thought of as capturing local information in the adjacency matrix of a weighted graph. Affinity matrices are sometimes colloquially referred to as \enquote{kernel} matrices, although in general they may fail to be positive semidefinite and thus are not strictly kernels. Pairs of points in close proximity (as measured in the ambient space) are assigned a high affinity, and pairs that are far away are assigned a low (or zero) affinity. One of the most common constructions is to build $W$ such that each point $x_i$ is connected by an edge to its $k$ nearest neighbours. The goal of $W$ therefore is to encode only \emph{local} relationships: since by assumption $\MC$ is smooth, the ambient Euclidean metric in $\mathbb{R}^p$ locally agrees with the Riemannian metric. Following this reasoning, the weighted graph generated by $W$ is a reasonably faithful representation of the latent manifold $\MC$.

We make note that affinity matrices, more generally, are employed ubiquitously throughout statistical and machine learning tasks which may not strictly fall under the manifold learning scenario that we introduced. For instance, $k$-nearest neighbour graphs are central to popular general-purpose \emph{dimensionality reduction} methods such as t-SNE \citep{van2008visualizing}, UMAP \citep{mcinnes2018umap}, and related algorithms \citep{agrawal2021minimum, jacomy2014forceatlas2}. These methods aim to produce low-dimensional embeddings that are suitable for human visualisation and interpretation, rather than to learn an underlying manifold in its strict sense. In such settings, affinity matrices still serve to encode notions of ``local'' pairwise relationships (or, put differently, a weighted graph spanning the dataset) which are used to inform a learning objective, although the connection to concrete mathematical objects becomes more tenuous.

\subsubsection{Laplacian operators} A key characterisation of the matrix $W$ is its that it defines a (discrete) linear operator acting on vectors supported on $\XC$, the set of points sampled from the manifold $\MC$. 
Under the interpretation of $W$ as a weighted graph, there is naturally the notion of a (combinatorial) graph Laplacian matrix:
\[
    L = D - W, \quad D = \diag(W\1). 
\]
The eigenvectors of $L$ may then used as basis functions to represent observed data in a lower dimensional space that is adapted to the geometry of the latent manifold, as in \cite{coifman2006diffusion} for instance. More generally, we consider the construction of a family of discrete linear operators on square-integrable functions supported on $\XC$ of the type
\begin{equation}
  \label{eq: DiscLaplOp}
  \Delta^D: L^2(\XC) \to L^2(\XC), \quad (\Delta^D g)(x_i) := \sum_{j=1}^N \overline{W}_{ij} \left(g(x_i) -g(x_j)\right),
\end{equation}
where the $\overline{W}_{ij} = \overline{W}_{ji} \ge 0$ are \emph{normalised} affinities obtained from $W$ (see Section \ref{sec:laplacian_normalisation}).
Treated as a discrete analogue to the Laplace--Beltrami operator \citep{von2008consistency}, the various flavours of the graph Laplacian enjoy a wealth of interpretations and uses. Indeed, the corresponding Dirichlet form is fundamental in the computation of eigenfunctions via the Courant\textemdash Fisher\textemdash Weyl min-max principle \citep[Section~III.1]{bhatia2013matrix}. Moreover, the Laplacian eigenbasis is the one appropriate for representing functions supported on graphs, leading to the vast toolbox of spectral methods for clustering and embedding \citep{von2007tutorial, belkin2003laplacian}.
As another example, the discrete Laplacian is connected to random walks and diffusion processes on graphs and discrete manifolds, yielding a discrete heat kernel from which geodesic distances can be estimated by utilising the connection to heat propagation \citep{crane2013geodesics}, a fact that has been utilised in applications to mass transport on manifolds \citep{solomon2015convolutional}. Laplace-based methods are foundational to applications in computer geometry \citep{levy2006laplace,reuter2006laplace}, and more generally in machine learning applications where Laplacian regularisations are used to adapt models to manifold or graph structures in datasets \citep{belkin2006manifold, cai2008non}.

\subsubsection{Normalisation of the Laplacian matrix} \label{sec:laplacian_normalisation}
In practice, the sampled data $\{ x_i \}_{i = 1}^N$ may be subject to variable sampling density and noise in the high-dimensional ambient space. The Laplacian matrix can be normalised in various fashions in order to account for the effects of variable sampling density as well as to have desirable properties for subsequent analyses. In fact, it is well known that appropriate normalisation is crucial for achieving good downstream results, as well as for there to be a connection to the continuum setting \citep{von2008consistency, ting2011analysis, coifman2006diffusion}. Laplacian eigenmaps \citep{belkin2003laplacian} and diffusion embeddings \citep{coifman2005geometric} are two closely related examples of classical nonlinear manifold learning approaches that utilise the spectral decomposition of a normalised Laplacian. 
Two of the most common normalisation choices are the symmetric normalisation,
\begin{equation}
  L^{(s)} = I - D^{-1/2} W D^{-1/2},
  \label{eq:norm_symm}
\end{equation}
which preserves the symmetric positive semidefinite property, and the random-walk or Markov normalisation
\begin{equation}
  L^{(rw)} = I - D^{-1} W,
  \label{eq:norm_rw}
\end{equation}
which affords a probabilistic interpretation as the generator of a diffusion process. A particularly interesting approach is that of \citet{coifman2006diffusion}, where a family of anisotropic diffusions are constructed by normalising a kernel matrix by a positive power $\alpha$ of its row sums. The normalised kernels are shown to have different interpretations for varying choices of $\alpha$, and for an appropriate value of $\alpha$ the authors demonstrate that the effect of sampling density is decoupled from geometry. 

\subsubsection{Bistochastic normalisations of affinity matrices}

Significant interest has been devoted to \emph{symmetric} and \emph{bistochastic} normalisations of affinity matrices which enjoy both self-adjointness as well as the probabilistic interpretation related to random walks \citep{marshall2019manifold, landa2021doubly}. This is particularly favourable for the manifold learning setting, as it mirrors the fact that the Laplace-Beltrami operator itself is both self-adjoint and generates the heat kernel on appropriate manifolds \citep[Chapters 3 \& 7]{grigoryan2009heat}. 
The symmetric normalisation scheme \eqref{eq:norm_symm} can be adapted to the bistochastic case: given an affinity matrix $W_\varepsilon$, a classical result on scalings of strictly positive matrices \citep{sinkhorn1964relationship} assures us that there exist multiplicative scaling factors $d_i > 0$ such that
\begin{equation}
  W^{(d)}_\varepsilon = \diag(d) W_\varepsilon \diag(d)
  \label{eq:bistochastic_scaling}
\end{equation}
is bistochastic \citep{landa2021doubly, zass2006doubly, van2024snekhorn}
where the $d_i > 0$ are a vector of multiplicative scaling factors applied to both the rows and columns of $W$. As discussed in their article, the normalised kernel $W^{(d)}_\varepsilon$ can be characterised as an \emph{information projection} of the unnormalised kernel $W_\varepsilon$:
\begin{equation}
  W^{(d)}_{\varepsilon} = \min_{W = W^\top, W\ones = \ones, W \ge 0} \KL(W | W_\varepsilon). 
\end{equation}
We defer further details on this projection and the structure of its solution to a unified description in Section \ref{sec:bistoch_projections}.  

One consequence of normalisation by multiplicative scalings is that the entries of the resulting kernel matrix remain strictly positive and so the graph is fully connected, which is at odds with the fact that manifolds are only \emph{locally} homeomorphic to the Euclidean space. Motivated by applications to spectral clustering, \citet{zass2006doubly} proposed to consider the projection of an affinity matrix sought with respect to the Frobenius norm, which induces sparsity in the projected matrix:
{
  \begin{equation}
    W_\eps^{(f)} = \min_{W = W^\top, W\ones = \ones, W \ge 0} \| W - W_\varepsilon \|_F^2.
  \end{equation}
  A consequence of the structure of this problem (see Section \ref{sec:bistoch_projections}) is that the optimal $W_\eps^{(f)}$ does \emph{not} have the form of a matrix scaling, i.e., \eqref{eq:bistochastic_scaling}. 
  An alternating projection method was proposed by \citet{zass2006doubly} to solve this problem, and empirically the advantage of using the Frobenius projection was demonstrated in clustering tasks. 
}

\subsection{Contributions: from bistochastic normalisations to sparse regularised optimal transport}
\label{sec: MainContrib}
In this work, we:
\begin{itemize}
    \item suggest a manifold learning procedure based on a new variation of quadratically regularised optimal transport that generalises the problem of bistochastic kernel scaling; 
    \item investigate its theoretical properties: we show its robustness to non-uniform noise, study the limit of discrete operators, and exhibit a similarity with the solution of the porous medium equation;
    \item exhibit its performance in simulated examples as well as in real-data applications.
\end{itemize} 
We briefly describe the above in the following subsections, while the corresponding results are developed in Sections~\ref{sec: Theory} and~\ref{sec: Perf}.  

The main idea of the paper is to use the \emph{bistochastic} and \emph{sparse} affinity matrix $W_\eps^{(f)}$ obtained as the solution to a quadratically regularised optimal transport problem (see \eqref{eq: Primal}) for constructing a weighted neighbourhood graph in statistical settings. In particular, we advocate for its use in manifold learning. With this aim in mind, the spectral decomposition of $W_\eps^{(f)}$ and its properties are of particular interest.  In practice, numerical computation of $W_\eps^{(f)}$ can proceed using a variety of methods: a variant of the Newton-type algorithm proposed by \citet{lorenz2021quadratically} can be relied on (see Algorithm~\ref{alg:lorenz_symmetric}). Compared to the alternating projections approach of \citep{zass2006doubly}, it converges in far fewer iterations and the key computation for each step involves iterative solution of a sparse linear system. Furthermore, the sparsity of the solution can be leveraged for further speedups when input sizes are large (see Algorithm~\ref{alg:activeset}). As shall be seen in Section~\ref{sec: Perf}, the proposed method performs well in a variety of settings. We believe this to be the result of the \enquote{local-to-global} character of the construction. Indeed, as the resulting plan is sparse, with each node only connected to its neighbouring points, the method constructs neighbourhoods capturing local information. Yet, as the optimisation problem involves the entire distance matrix and the global bistochastic constraint, the optimal dual potential incorporates more than just local information. In this sense, unlike approximation methods such as $k$- or $\epsilon$-nearest neighbours which ensure local properties only, we do not sacrifice global information in favour of local information.

{

\subsubsection{Bistochastic projections of affinity matrices}\label{sec:bistoch_projections}

We begin by stating a general bistochastic matrix \emph{projection} problem: denote by $\KC$ the set of bistochastic $N \times N$ matrices:
\begin{equation}
  \KC = \{ A \in \mathbb{R}^{N\times N} : A_{ij} \ge 0, A = A^\top, A\ones = A^\top \ones = \ones \},
\end{equation}
and this is compact and convex. 
Let $(A, B) \mapsto \mathcal{D}(A | B) \ge 0$ be a chosen divergence between matrices. Given a symmetric matrix $W$, we seek its projection onto $\KC$ with respect to $\mathcal{D}$:
\begin{equation}
  \operatorname{Proj}^{\mathcal{D}}_{\KC} (W) = \argmin_{A \in \KC} \mathcal{D}(A | W).
  \label{eq:bistoch_projection}
\end{equation}
Note here that, for general divergences $\mathcal{D}$, we do not necessarily require that the entries $W$ be positive: indeed in Section \ref{sec:generalized_projections} we will find that in the case we consider, the ``right'' matrix $W$ to project has negative entries.

We point out that the two problems studied by \citep{zass2006doubly} correspond to the cases where the divergence $\mathcal{D}(A | B)$ is chosen as the Kullback-Leibler divergence and Frobenius norm, and $W$ is allowed to be any affinity matrix. \cite{landa2021doubly} also consider the Kullback-Leibler divergence, but with two key differences: in their analysis, $W$ has the specific form of a Gaussian kernel, i.e., $W = W_\varepsilon = e^{-C/\varepsilon}$ where $C$ is a matrix of squared Euclidean distances. Additionally, they restrict to the set of \emph{hollow} bistochastic matrices.
As we will see, these two choices have a clear interpretation that applies beyond the specific case of the information projection. In particular, it remains relevant in the case of the sparsity-inducing Frobenius projection, which is the one we study.
The variant of problem \eqref{eq:bistoch_projection} studied by \citet{landa2021doubly} can be written as 
\begin{align}
  \begin{split}
  W_\varepsilon^{(d)}(C) &:= \argmin_{A \in \KCo} \KL( A | W_\varepsilon ) \\
                         &= \argmin_{A \in \KCo} \frac{1}{\varepsilon} \langle A, C \rangle + \KL(A | \1 \1^\top), 
  \end{split}\tag{EOT}\label{eq:kl_projection}
\end{align}
where $\KL(A | W) = \sum_{ij} A_{ij} \log( A_{ij} / W_{ij})$ is the Kullback-Leibler divergence between non-negative matrices treated as measures. We denote by $\KCo$ the set of \emph{hollow}, symmetric, bistochastic kernel matrices, i.e., those with all zeros on the diagonal:
\begin{equation}
    \KCo = \KC \cap \left\{ A \in \mathbb{R}^{N \times N} : \diag(A) = 0 \right\} . 
\end{equation}
The requirement on the diagonal removes self connections, which was shown to be crucial for rendering the procedure robust to heteroskedastic noise in \citep{landa2021doubly}, the reason for this shall become clear from the explanations around equation \eqref{eq: high dim noise}. Let us also remark here that empirically, in practice, we find that this constraint is essential to achieving good results (see Section \ref{sec: Perf}). Another interpretation of the problem is as a symmetric \emph{optimal transport} problem with entropic regularisation, from which the scaling factors $d$ in \eqref{eq:bistochastic_scaling} arise as (exponentials) of the Lagrange multipliers for the row- and column-sum constraint. This problem can be solved efficiently in practice using the celebrated Sinkhorn matrix scaling algorithm \citep{cuturi2013sinkhorn}.

This will be our starting point: substituting the Frobenius norm in place of the $\KL$-divergence in \eqref{eq:kl_projection}, we arrive at a \emph{hollow, quadratically} regularised optimal transport problem
\begin{align}
  \begin{split}
    W_\eps^{(f)}(C) &:= \argmin_{A \in \KCo} \frac{1}{\varepsilon} \langle A, C \rangle + \| A \|_F^2 \\
                    &= \argmin_{A \in \KCo}  \| A + \varepsilon^{-1} C \|_F^2.
  \end{split}\tag{QOT} %
  \label{eq: QOTProb}
\end{align}
For $\varepsilon > 0$, the problem \eqref{eq: QOTProb} is the minimisation of a strongly convex function on the convex and compact set $\KC^\circ$. Whenever clear from context, we will drop the argument $C$ of $ W_\eps^{(f)}(C)$.
It is apparent from \eqref{eq: QOTProb} that this amounts to a bistochastic projection with respect to Frobenius norm of a negative distance matrix $-C/\varepsilon$. While this matrix itself is not a positive semidefinite kernel, 
since $C_{ij} = \| x_i \|^2/2 + \| x_j \|^2/2 - \langle x_i, x_j \rangle$ it is the sum of a Gram matrix and two rank-1 terms. It will be apparent in Section \ref{sec: RobandOpt} that the latter two terms are immaterial. The general problem of optimal transport with a quadratic regularisation was studied in detail by \citet{lorenz2021quadratically, essid2018quadratically}, and for completeness we provide a discussion in Appendix \ref{sec:general_OT}. Similarly to \citet[Equation D]{lorenz2021quadratically}, the primal problem \eqref{eq: QOTProb} admits as its counterpart the following dual problem, formulated in terms of a Lagrange multiplier $u$,
\begin{align}
\label{eq: DualLor}
  \max_{u} \: \langle u, \1 \rangle - \frac{1}{4 \varepsilon} \left\| [u \1^\top + \1 u^\top -  C]_+ \odot (\1 \1^\top - I) \right\|_F^2,
\end{align}
where $[x]_+$ is the positive part of $x$ applied elementwise, as follows from  Lemma~\ref{eq: DualHollow}. This is a convex non-smooth problem: denoting by $u^\star$ an optimal solution to \eqref{eq: DualLor}, the projection can be recovered by 
\begin{align}
\label{eq: optimalPlan}
    W_\eps^{(f)}(C) = \frac{1}{\varepsilon} [u^\star \1^\top + \1 {u^\star}^\top - C]_+ \odot (\1\1^\top - I).
\end{align}
From this characterisation of the projection, it becomes apparent that some off-diagonal entries may be identically zero -- whenever $u^\star_i + u^\star_j \leq C_{ij}$. 

The sparsity of the projected affinity matrix $W_\eps^{(f)}$ leads to practical benefits for downstream tasks such as finding eigenvalue-eigenvector pairs, which partly makes up for the fact that the solution of the quadratically regularised problem \eqref{eq: DualLor} is slightly more expensive than the entropically regularised one. Additionally, as pointed out by \cite{lorenz2021quadratically}, the algorithm for computing the projection is stable with respect to the parameter $\varepsilon$. This is in contrast to the entropic case, where the na\"ive implementation of Sinkhorn iterations may suffer from numerical stability issues \citep{schmitzer2019stabilized}, and require stabilising variations that have since been developed \cite[Eq.~2.41]{feydy2020geometric}. We also note that a more efficient algorithm exists for the symmetric case \citep{knight2014symmetry}. 

\subsubsection{Generalised bistochastic information projections}\label{sec:generalized_projections}
In the above, we have discussed two different approaches for finding bistochastic projections of kernel matrices. One is in terms of an information projection of a Gaussian kernel matrix which amounts to multiplicative scalings of rows and columns, and the other is a projection of minus the matrix of pairwise distances in the Frobenius norm. We now show that both manifest as two special cases of a generalised information projection of the \emph{same} Gram matrix. While we defer details to Appendix \ref{sec:general_OT}, note that for $q \in [0, 1)$ a generalised family of exponential and logarithm functions can be defined \citep{bao2022sparse} by
\begin{equation}
  \exp_q(x) = [1 + (1-q)x]_+^{1/(1-q)}, \quad \log_q(x) = (1 - q)^{-1} (x^{1-q} - 1).
\end{equation}
We remark that as $q \to 1$ we recover the conventional exponential and logarithm. 
This leads to an analogue $\Phi_q$ of the relative entropy functional being defined on matrices, and its generator $\phi_q$ in the sense of the Bregman divergence, given by
\begin{align}
    &\Phi_q(A | B) = \frac{1}{2-q} \sum_{i, j} \left[ A_{ij} \log_q A_{ij} - A_{ij} \log_q B_{ij} - A_{ij} B_{ij}^{1-q} + B_{ij}^{2 - q}\right], \\
    & \phi_q(x) = \frac{1}{2-q} ( x \log_q x - x),
\end{align}
Indeed when $q = 0$ this corresponds to the squared Frobenius norm $\frac{1}{2} \| A - B \|_F^2$, and when $q \to 1$ one recovers the KL-divergence. We find that a family of bistochastic projections for $q \in [0, 1)$ can be written
\begin{equation}
    \argmin_{A \in \KCo} \varepsilon \Phi_q^\star\left( \frac{-C}{\varepsilon}, \nabla \phi_q(A) \right),
\end{equation}
where $\Phi^\star$ denotes the Legendre dual of $\Phi$. When $q = 0$ we recover the Frobenius norm projection of $-C/\varepsilon$, and as $q \to 1$ we recover the information projection of $\exp(-C/\varepsilon)$. 

Curiously, the Gaussian kernel arises naturally from the linear kernel in this framework since $\phi_q(x) \to x \log x - x$ as $q \to 1$, and $\nabla \phi_q^{-1}(y) = e^y$. This suggests that the \emph{linear} kernel is the natural one when using a Frobenius projection, and using alternative kernels amount to mapping first to a reproducing kernel Hilbert space (RKHS). Much in the same way, for the information projection the Gaussian kernel is the natural one, in direct correspondence to the data space. We emphasise that this insight is in contrast to previous works \citep{zass2006doubly, ding2022understanding, landa2021doubly, landa2022robust, marshall2019manifold} which are less discriminate in the choice of kernel matrix; for instance considering projections of an arbitrary kernel matrix, or the Frobenius projection of a Gaussian kernel matrix. This is of particular importance for manifold learning: in contrast to clustering applications \citep{zass2006doubly} where input mappings to an RKHS are potentially of no consequence to the clustering results, for manifold learning one should use the linear (Euclidean) kernel since it locally coincides with the metric of the embedded manifold. 

\subsubsection{A spectral motivation for projections in Frobenius norm}
The bistochastic projection of an affinity matrix $W$ with respect to the Frobenius norm can be understood as finding the best bistochastic matrix $A$ in terms of the Davis-Kahan spectral upper bound. Considering the $k$ leading eigenvalues of $A$ and $W$ respectively (and their eigenvectors), the Davis-Kahan theorem \citep[Theorem 2]{yu2015useful} states that 
\begin{equation}
    \| \sin \Theta(V_k(A), V_k(W)) \|_F \leq \frac{2 \| W - A \|_F }{ \lambda_k(W) - \lambda_{k+1}(W)}.
\end{equation}
In the above, the left-hand side is the norm of the matrix of (sines of) pairwise principal angles. This measures how close the leading $k$ eigenspaces of $A$ and $W$ are. For $1 \leq k \leq N$, $\lambda_k(W)$ is the $k$th eigenvalue (in descending order) of an affinity matrix $W$, and $V_k(W)$ is the matrix with orthonormal columns formed from the leading $k$ eigenvectors of $W$. In our context $W$ is a fixed noisy affinity matrix determined by our data, and so only the numerator is of relevance. Among all bistochastic affinity matrices we consider, the one that is most faithful to the spectrum of the observed affinities $W$ as measured by the Davis-Kahan upper bound is therefore the one that is closest in Frobenius norm. This insight agrees with our empirical observations that the Frobenius projection has better spectral properties compared to the information projection. Note further that, as we shall see in the robustness section, the method is invariant with respect to some types of noise, implying that one indeed finds a faithful representation of the unknown truth.

\subsubsection{Geometric view on the optimisation problem}
In the field of manifold learning, the local covariance matrix is paramount \textemdash see \citet{malik2019connecting} and the references therein. 
Indeed, it captures the geometry of the manifold at a point. Its truncated inverse can then be used in a Mahalanobis distance to approximate the geodesic distance on the manifold\footnote{
More precisely, denoting by $C_\eps(x)$ the $\eps$ local covariance matrix centred at $x$ and by $T_\alpha$ the maps sending a matrix $n\times n$ $A$ onto $U_A (\Lambda(A)^{-1} \odot \diag (1_\alpha, 0_{n-\alpha} ) )U_A^\top$ $(\alpha\le n)$, the geodesic distance can be approximated by \[
(x-y)^\top (T(C_\eps(x)) + T(C_\eps(y)) (x-y)/2, 
\]where $\alpha$ and $\eps$ have to be chosen properly.
}. 
The empirical local covariance matrix at $x_j$ can be defined, for a \enquote{neighbourhood selector}  matrix $A$
 as $C_A(x_i):=(X-x_i)^\top A^\top A (X-x_i)$.

For example, \citet{malik2019connecting} consider a matrix containing zeros at the places corresponding to elements far from $x_j$, while having value $(\1^\top A \1)^{-1/2}$ otherwise.  
The problem \eqref{eq: QOTProb}, can be rewritten using this notion of a neighbourhood selector
as  
\begin{align}
\inf_{A \in \KCo} \sum_{i=1}^n \sum_{j=1}^n  \tr \left( \lVert x_i -x_j \rVert^2 A_{i,j} \right) + \eps \sum_{i=1}^n\sum_{j=1}^n A_{ij}^2
=\inf_{A \in \KCo} \sum_{i=1}^n   \tr \left( C_A(x_i) \right) + \eps \sum_{i=1}^n\sum_{j=1}^n A_{ij}^2
\end{align}
so that the problem can be interpreted as finding a matrix $A$ that balances the mean empirical local variances\footnote{Recall that the sum of eigenvalues of a covariance matrix is the variance.} while controlling for locality. 
The fact that the resulting matrix is sparse is then a desirable/necessary feature. Indeed, a Riemannian manifold is only locally homeomorphic to the Euclidean space and dense matrices $A$ will likely give rise to corrupted/biased local correlation matrices. This explains why the density of points must therefore be taken into account as well as the local geometry and why the proposed method is appealing. Importantly, the relationship above is true only if one takes the squared Euclidean distance as a cost function. This explains why it is not indicated to take any arbitrary affinity matrix to carry out manifold learning; for clustering, the situation is different.

}

\subsubsection{Sketch of the theoretical results}
\label{sec: SketchResults}

We will first set up a few more notations. Let $\mu$ be an embedded empirical or continuous measure on the manifold, and denote by $C(x, y) = \frac{1}{2} \| x - y \|_2^2$ the squared Euclidean distance function in $\reals^p$. To be consistent in both discrete and continuous cases, we write
\begin{equation}
    \label{eq: Primal}
     \pi^\star = \argmin_{\pi \in \KCo}  \left\lVert \pi - \tfrac{1}{\eps} (-C)\right\rVert_{\mu \otimes \mu}^2  = \argmin_{\pi \in \KCo}\: \langle C, \pi \rangle_F + \frac{\eps}{2} \lVert \pi \rVert_{\mu \otimes \mu}^2
\end{equation}
where $\KCo = \{ \pi : \pi(A \times \XC) = \pi(\XC \times A) = \mu(A), \frac{\diff \pi}{\diff \mu \otimes \diff \mu}(x, x) = 0 \: \forall x \}$ denotes the set of symmetric positive measures that have marginals $(\mu, \mu)$. We remark that the hollowness condition from the discrete setting vanishes in the continuous setting, since then a pointwise requirement on the density of $\pi$ along its \enquote{diagonal} has no effect. We write $\| \pi \|_{\mu \otimes \mu}^2 = \langle \frac{\pi}{\mu \otimes \mu}, \pi \rangle$ to mean the squared norm of $\frac{\diff\pi}{\diff\mu \otimes \diff\mu}$ in the space $L^2(\mu \otimes \mu)$. See \citet{nutz2024quadratically} for a complementary treatment.%

The dual problem and the primal-dual relationship at optimality are then
\begin{align} 
  \label{eq: Dual}
  \begin{split}
  &\max_{u}\ \langle u, \mu \rangle - \frac{1}{4 \varepsilon} \left\| [u \oplus u - C]_+ (\mu \otimes \mu) \right\|_{\mu \otimes \mu}^2, \\ 
  & \pi^\star = \frac{[u^\star \oplus u^\star - C ]_+}{\varepsilon}(\mu \otimes \mu).
  \end{split}
\end{align}
Therefore, the optimal potential $u^\star$ must satisfy
\begin{equation}
\label{eq: DualOpt}
    \frac{1}{\varepsilon} \int_{\XC} [u^\star(x) + u^\star(y) - C(x, y)]_+ \: \diff \mu(y) = 1, \quad \forall x \in \XC 
  \end{equation}
Notice that for the discrete setting one can straightforwardly recover the equations \eqref{eq: DualLor}, \eqref{eq: optimalPlan} from their continuous equivalents \eqref{eq: Dual}, \eqref{eq: DualOpt} by setting $\mu = \1$, i.e., the counting measure.

The main theoretical contributions are robustness to heteroskedastic noise, proving the convergence of a first order approximation of the affinity kernel and exhibiting a link with the porous medium equation. Let us start with robustness to heteroskedastic noise. Following \citet[Theorem 3]{landa2021doubly} and recalling the setup of Section \ref{sec: ManiGraph}, we consider a fixed set of points $\{ y_i \}_{i = 1}^N$ on the $d$-dimensional manifold $\MC$ that is then embedded in the $p$-dimensional Euclidean space to yield $\{ x_i \}_{i = 1}^N$. We consider the setting where the ambient dimension increases $p \to \infty$. We next let these embedded points be contaminated by heteroskedastic additive noise, i.e., 
\begin{equation}
\label{eq: HetNoise} 
    \tilde{x}_i= x_i + \eta_i,  
\end{equation}
with $\expec[\eta_i]=0$ and $\expec[\eta_i \eta_i^\top] = \Sigma_i$. 

Assume further that although the set of points prior to embedding is fixed, there exists some constant $\kappa > 0$ such that $\| x_i\| \le \kappa, i =1,\ldots, n$, and that $\| \Sigma_i \|_2 \leq \kappa_\eta p^{-1}, $
which means that the noise does not concentrate too much in a particular direction. For some models, normalisation of the data by $\sqrt{p}$ also ensure that the latter condition holds \citep[Remark~2]{landa2021doubly}.
 Then for $i \ne j$, 
\begin{equation}
\label{eq: high dim noise}
    \| \tilde{x}_i - \tilde{x}_j \|^2 \xrightarrow{\mathrm{P}} \expec \| \eta_i  \|^2 +\| {x}_i - {x}_j\|^2+ \expec \| \eta_j  \|^2, \qquad \text{ as } p \to \infty.
\end{equation}
This type of noise contamination is also well cancelled by the quadratically regularised optimal transport plan. Indeed, if $u^\star$ is the optimal dual transport potential
for the uncontaminated setting, 
\[
\tilde{u}_i^\star:= u_i^\star + \expec \| \eta_i  \|^2
\]
is an approximate optimal dual potential for the cost matrix $C_{ij} = \| \tilde{x}_i - \tilde{x}_j \|^2$ for $p$ large.
This intuition can also be quantified: this is the statement of next proposition, whose proof is given in Appendix~\ref{sec: AppLem}.
\begin{proposition}
\label{prop: RobHetNoise}
Consider the heteroskedastic noise model introduced in \eqref{eq: HetNoise}. Denote by $\widetilde{C}$ the squared Euclidean cost matrix based on the corrupted points $\tilde x_i$. Then, 
\[
    \lVert W_\eps^{(f)}(C) - W_\eps^{(f)}(\widetilde C) \rVert \le \mathcal{O}_p(p^{-1/2}),
\] 
where $\mathcal{O}_p$ refers to stochastic boundedness. 
\end{proposition}
We add that this robustness to heteroskedastic noise is clearly visible in practice, see Figure~\ref{fig:figure_spiral} below.

The main theoretical result of the paper, contained in Theorem~\ref{thm: MainConv}, pertains to the limit of the discrete Laplace operator based on the optimal transport plan and can be understood as follows. We refer to Section~\ref{sec: ConvOp} for the full picture.
Before stating the theorem let us introduce 
\bgroup
 $\overline{W}^{(f)}$. This is a first-order approximation of the optimal transport plan ${W}^{(f)}$ that combines Lemma~\ref{lem: OptPot} below and the dual representation \eqref{eq: optimalPlan}.
 Set 
 \begin{equation}
  \overline{W}_\eps^{(f)}(C) := \frac{1}{\varepsilon} \left[ 2C_d \: \eps^{\frac{2}{d+2}}N^{-\frac{4}{(d+2)}} ( \1\1^\top) - C\right]_+ 
 \end{equation} 
 where 
 \begin{equation}
 C_d:= \left(  \frac{\vol(\MC) }{\lvert S^{d-1}\rvert} \ \frac{d(d+2)}{2}\right)^{\frac{
2}{d+2}} .
 \end{equation}

    \egroup
    \vspace{12pt}

\noindent
\textbf{\enquote{Theorem} (Sketch of Theorem~\ref{thm: MainConv}).} \textit{ Consider a random sample $\{X_i\}_{i=1}^N$ independently and uniformly distributed on a smooth manifold $\MC$ embedded in $\reals^p$ and a given point $X_0$ of the embedded manifold.  Then, there exists a Laplace-type operator $\Delta^\infty$ and,  for any fixed function $g\in C^2$, it holds that
\begin{equation}
- 2 (N+1) K_{\eps, N}^{-1} \sum_{j=0}^N \overline{W}^{(f)}_{0,j} \big(g(X_0)-g(X_j)\big) \overset{L^2}{\to} \Delta^\infty g(X_0), \text{ as } N \to \infty, \eps \to 0,
\end{equation}
 where $K_{\eps, N}$ is a well-chosen sequence depending on $\eps$, $N$ and choosing $\eps$ as an appropriate function of $N$.
}

\begin{remark}[Convergence rates]
    It is important to note that, as the proof relies on quantification of the bias and the variance, convergence rates could be made explicit from the proof.
\end{remark}

Our final observation is that plugging the first-order optimal dual potentials from Lemma~\ref{lem: OptPot} in the representation \eqref{eq: optimalPlan} gives a kernel which behaves like the Barenblatt\textemdash Prattle solution of the porous medium equation, see Section~\ref{sec:PME} for further details.

\section{Theoretical contributions}
\label{sec: Theory}

Let us now describe further the manifold setting that we consider for random samples. Let $X$ be a $p$-dimensional random variable whose range is supported on $\MC$. Recall that,  $\iota$ denotes the isometric embedding of $\MC$ into $\reals^p$ and $\iota_*$ its differential. We use $\diff\!\vol$ to denote the Riemannian volume form of the Riemannian manifold $(\MC,g)$ of interest. Let us further assume for simplicity that $X$ has a uniform distribution on the manifold, i.e., the density of $X$ with respect to $(\iota \vol)$ is $ P(x) = \vol^{-1}(\MC), \forall x \in \iota(\MC)$. Whenever we write $\expec$ or $\var$, unless explicitly stated, denoted we mean it to be with respect to $P\ (\iota  \vol)$.

Let us denote by $s(x)$, the scalar curvature of the manifold at $x$ and by $\mathbb{I}_x$ the second fundamental form of the isometric embedding $\iota$ at $x$.  

Further, referring to the tangent space $T_y\MC$ at $y \in \MC$, we denote by $\iota_*T_y\MC$ the corresponding embedded tangent space in $\reals^p$. Recalling that the intrinsic dimension of $\MC$ is $d$, this embedded tangent space is thus $d$-dimensional.
The scalar curvature, as its name suggests, expresses through a single number the curvature of the manifold; more precisely it is defined as the trace of the Ricci curvature tensor.  
With this notation at hand, we recall that the second fundamental form is a symmetric bilinear form from 
$T_y\MC \times T_y\MC$ to $(\iota_*T_y\MC)^\perp$. It encode relevant geometric information about the manifold.  In the particular case that interests us, we only need the second fundamental form evaluated on the diagonal of $T_y\MC \times T_y\MC$. In that case, it can be understood as evaluating the fluctuation of a tangent vector $v\in T_y\MC$ in the direction $v$ (covariant derivative) and retaining only the components of that fluctuation that are normal to the tangent space.

In the sequel, $\nabla$ will denote the covariant derivative while $\Delta$ will be the Laplace-Beltrami operator.
Further, set for each $x\in \MC$,
\[
    \omega(x)=\frac{1}{\lvert S^{d-1}\rvert } \int_{S^{d-1}} \lVert \mathbb{I}_x(\theta, \theta) \rVert^{2} \diff \theta,
\]
as well as 
\[
    \mathfrak{N}(x)=\frac{1}{\lvert S^{d-1} \rvert} \int_{S^{d-1}}  \mathbb{I}_x(\theta, \theta)  \diff \theta,
\]
where $S^{d-1}$ denotes the $(d-1)$-dimensional unit sphere in $T_x \MC$.

\bgroup

The quantity $\omega(x)$ can be understood as some directionally averaged measure of curvature at $x$, while 
$ \mathfrak{N}(x)$ is an average across all directions of the tangent space  of the second fundamental form evaluated at the pair of unit vectors $(\theta, \theta)\in T_y\MC \times T_y\MC$. The vector  $\mathfrak{N}(x)$ can be the zero vector at saddle points of the manifold if the curvatures cancel in a suitable way, for instance.
\egroup

For the sake of simplicity, let us make the following assumptions. 

\begin{assum}
\label{assum: Rot}
The manifold $\MC$ is correctly shifted and rotated so that $\iota_* T_{x_0}\MC$ is spanned by $e_1, \ldots, e_d$.
\end{assum}

\begin{assum}
\label{assum: Assum2}
The manifold is properly rotated and translated so that $e_{d+1}, \ldots, e_p$ diagonalise the second fundamental form $\mathbb{I}_{x_0}$. 
\end{assum} 
These two conditions are not particularly important, they just help simplify both notation and result statements.
Under these assumptions, we use the notation $\llbracket v_1, v_2\rrbracket$ for the vector $v$ whose $d$ first components are the vector $v_1$ and its $p-d$ last components are the vector $v_2$. The $p\times r$ matrix $\tilde{J}_{p,r}$ is then defined as 
\[
\tilde{J}_{p,r}:= \begin{pmatrix}
0_{p-r\times r} \\ I_{r\times r}
\end{pmatrix}
\]
In what follows, for some $\eta$ sufficiently small, define the fattened manifold in the embedded space by $\mathcal{N}$, i.e., the set $\mathcal{N}:= \{ x\in \reals^p: \inf_{y\in \MC} \lVert x-\iota (y) \rVert \le \eta\}$. 
We recall that the Laplace operator for a function $g$ on a certain embedded smooth manifold $\iota(\MC)$ is defined for $p \in \iota(\MC)$ as
\[
\Delta g( p) : = (\Delta_{\reals^p} g_{\text{ext}}) (p) 
\]
 with $\Delta_{\reals^p}$ the usual Laplace operator in the Euclidean space, $g_{\text{ext}}(x):= g ( \pi_\MC(x))$, $x\in \mathcal{N}$, and $\pi_\MC$ projects $\mathcal{N}$ onto the manifold. For the entire paper, \enquote{$r$ sufficiently small} must be understood as $r$ being smaller than the injectivity radius of the manifold.

\subsection{Nearly optimal potentials}
\label{sec: NearOptPot}

We now turn to the study of the optimal rates for the potentials in the discrete case.
Let us slightly change the setting and consider a sample of size $N+1$ where one point, $x_0\in \MC$, is fixed and the remaining ones are an i.i.d.\ random sample on the manifold. Set $X_0=\iota(x_0)$. We relabel the sampled points so that $\lVert X_0 - X_1 \rVert^2 \le \lVert X_0 - X_2 \rVert^2\le \ldots \le \lVert X_0 - X_N \rVert^2$. 

\begin{remark}[No loss of generality in choosing $X_0$]
\label{remark: FixPoint}
In the results below, the same analysis can be carried out for each point $X_i$. One can thus view our (convenient) choice of working with one distinguished, deterministic point as a conditioning on an arbitrary $X_i$. Still, as the expectations of the quantities for $X_0$ fixed are constants with uniformly decaying terms, the reasoning would apply for each $X_i$ using the tower property of conditional expectation.   
\end{remark}

\begin{lemma}
\label{lem: OptPot}
The optimal potential admits the following behavior.
\[
 u^\star(x) \sim K_{\eps, N}:=\left(  \frac{\vol(\MC) }{\lvert S^{d-1}\rvert} \ \frac{d(d+2)}{2}\right)^{\frac{
2}{d+2}} \eps^{\frac{2}{d+2}}N^{-\frac{4}{(d+2)}} = C_d \: \eps^{\frac{2}{d+2}}N^{-\frac{4}{(d+2)}},
\]
for all $x \in \iota\mathcal{M}$.
\end{lemma}

These results provide a reasonable \textit{ansatz}, still these are only approximations. We now assess the quality of this first order approximation of the solution by evaluating how the dual constraints are fulfilled when plugging-in the first-order approximation of the solution.

\subsubsection{Validity of the derived finite sample rate}
\label{sec: Validity}
We will use the function
\[
 f(y) = C_d \eps^{2/(d+2)} N^{-4/(d+2)} - \lVert \iota(x_0)- y \rVert^2  =  K_{\varepsilon, N} - \lVert \iota(x_0)- y \rVert^2 
\]
 and apply to the result of Lemma B.5 from \citet{wu2018think} to evaluate the constraints arising from the dual formulation of the problem, recall~\eqref{eq: optimalPlan} and~\eqref{eq: DualOpt}. Doing so,  one gets
\begin{align*}
&\expec \left( \frac{N+1}{\eps} \sum_{j=0}^N \left( K_{\eps,N} - \lVert X_j - \iota(x_0) \rVert^2 \right)_+ \right)  \\
&\qquad= \frac{N+1}{\varepsilon} \left\{ K_{\varepsilon, N} + N \expec \left[ f(X) \1\left\{ \| \iota(x_0) - X \| \leq K_{\varepsilon, N}^{1/2} \right\} \right] \right\} \\
  & \qquad = \frac{N+1}{\varepsilon} K_{\varepsilon, N} + \frac{N(N+1)}{\varepsilon} \frac{|S^{d-1}|}{d \vol(\MC)} K_{\varepsilon, N}^{1 + d/2} \\
    & \qquad\qquad+ \frac{N(N+1)}{\varepsilon} \frac{|S^{d-1}|}{d(d+2)\vol(\MC)} K_{\varepsilon, N}^{1 + d/2} \left[ -d + \frac{s(x_0) K_{\varepsilon, N}}{6} + \frac{d(d+2) \omega(x_0) K_{\varepsilon, N}}{24} \right] \\
    & \qquad\qquad+ \frac{N(N+1)}{\varepsilon} \Oh(K_{\varepsilon, N}^{\frac{5 + d}{2}} ) \\
      & \qquad = \frac{N+1}{\varepsilon} K_{\varepsilon, N} + \frac{N(N+1)}{\varepsilon} \frac{|S^{d-1}|}{ \vol(\MC)} K_{\varepsilon, N}^{1 + d/2} \left(\frac{1}{d} - \frac{1}{d+2} \right) + \Oh\left( \frac{N(N+1)}{\varepsilon} K_{\varepsilon, N}^{2 + d/2}  \right) 
\end{align*}

In the display above, terms have orders $\Oh(N \varepsilon^{-1} K_{\varepsilon, N}) = \Oh(\varepsilon^{-d/(d+2)} N^{(d-2)/(d+2)})$,  $\Oh(N^2 \varepsilon^{-1} K_{\varepsilon, N}^{2 + d/2})$ and we remark that $\Oh(N^2 \varepsilon^{-1} K_{\varepsilon, N}^{1 + d/2}) = \Oh(1)$. 
This latter term is the leading order. Then, to fulfil the constraint, we need that
\[ 
C_d^{\frac{2+d}{2}} \frac{|S^{d-1}|}{\vol(\MC)} \left( \frac{1}{d} - \frac{1}{d+2} \right) = 1, 
\]
so we have verified by this different approach that
\[
C_d=  \left( \frac{\vol(\MC)}{|S^{d-1}|} \frac{d(d+2)}{2}  \right)^{\frac{2}{d+2}},
\]
which matches with the expression above.
One gets that the chosen rates for the potential gives the correct constraint in expectation at the first order. 

 One can rewrite the conditions that $\Oh(\varepsilon^{-d/(d+2)} N^{(d-2)/(d+2)})= \oh(1)$ as 
 $\varepsilon^{-d} N^{(d-2)} \to 0 $ and  the condition $\Oh(N^2 \varepsilon^{-1} K_{\varepsilon, N}^{2 + d/2})= \Oh(K_{\eps,N})= \oh(1)$  as $\varepsilon^{2} N^{-4} \to 0$. Note that the latter condition was already somewhat required to apply  Lemma~B.5 from \citet{wu2018think}. 
Together, these results indicate that the asymptotic scaling on $\varepsilon$ is
 $N^{1 - 2/d} \ll \varepsilon \ll N^2$  for the constraints to be asymptotically fulfilled in expectation. 
 Note that our analysis shows that $\varepsilon$ needs not go to zero asymptotically but needs to increase. Rather, the need is for $\varepsilon$ to be asymptotically sufficiently small relative to $N^2$. See the remark below. 
 
 \begin{remark}
 In the developments above, we considered a classical optimal transport without the hollow constraint. The only difference is the term $\Oh(N \varepsilon^{-1} K_{\varepsilon, N})$, which is asymptotically negligible. 
 \end{remark}

\begin{remark}
\label{remark: Uniformly}
    In the developments above, the result holds uniformly in $x_0$ under quite mild assumptions as, for a closed\footnote{Recall that a manifold is closed if it is compact and without boundary. } and smooth manifold, the different kinds of curvatures appearing in the expansions are bounded, recall Remark~\ref{remark: FixPoint}.
\end{remark}
Thanks to Remark~\ref{remark: Uniformly}, one can derive that all the constraints will asymptotically be fulfilled in expectation when replacing the sum of optimal potentials by $K_{\eps,N}$. 

\begin{remark}
There is a difference in scaling between the discrete and continuous settings in our analysis \textemdash to get empirical input distributions that are consistent with the continuous setting in the limit of large $N$, for samples $X_1, \ldots, X_N$ we take $\hat{\mu} = N^{-1} \sum_i \delta_{X_i}$ as the corresponding empirical distribution. Suppose $\pi$ is an admissible coupling for such a discrete problem. Then $\pi$ is concentrated on the support of $\hat{\mu} \otimes \hat{\mu}$ and admits a density, $(\diff \pi / \diff \hat{\mu} \otimes \diff \hat{\mu})(X_i, X_j) = N^2 \pi_{ij}$. Then, note that the corresponding empirical entropy term would behave like
\begin{align*}
    H(\pi | \hat{\mu} \otimes \hat{\mu}) = \int \diff \pi \log\left( \frac{\diff \pi}{\diff \hat{\mu} \otimes \diff \hat{\mu}} \right) = \sum_{ij} \pi_{ij} \log\left( N^2 \pi_{ij} \right)
\end{align*}
Up to a constant, this is equal to the discrete entropy of $\pi$, i.e., $\sum_{ij} \pi_{ij} \log \pi_{ij}$. Thus, we expect no scaling behaviour between $N$ and the entropic regulariser. 

On the other hand, for the quadratic regulariser, one would have
\begin{align}
\label{eq: Norm}
    \| \pi \|^2_{\hat{\mu} \otimes \hat{\mu}} = \int \frac{\diff \pi}{\diff \hat{\mu} \otimes \diff \hat{\mu}} \diff \pi = \sum_{ij} \pi_{ij} \frac{\pi_{ij}}{N^{-2}} = N^{2} \| \pi \|^2_F. 
\end{align}
Thus, there is the presence of a factor $N^{2}$. This can be understood in that $N$ appears in the density of $\pi$ w.r.t. empirical product measure, which is lost as an additive term in the case of a log, but remains in the quadratic case. Thus, noting that $\varepsilon N^{-2} \| \pi \|_{\hat{\mu} \otimes \hat{\mu}}^2 = \varepsilon \| \pi \|_F^2$, it is apparent that the requirement that the effective epsilon $\tilde{\varepsilon} = \varepsilon N^{-2} \to 0$ in the continuous setting corresponds to $\varepsilon \ll N^2$ in the discrete setting. This is in agreement with the scaling we derived earlier. %

\end{remark}

\subsection{Robustness and potential optimality of the method}
\label{sec: RobandOpt}

A natural question arising is whether the approach we propose can be \emph{proven} to be optimal and if so, in what precise sense or context. Our first points address the robustness of the approach in the case of perturbations of the cost matrix, while the second provides a link to the problem of choosing a kernel in statistical density estimation or regression problems.

\subsubsection{Robustness}
\label{sec: Rob}
Robustness was already hinted at in the main contributions, see Section~\ref{sec: SketchResults}. We thus only provide additional elements and details here below.

\begin{lemma}
  \label{lem: MetProj}
  Consider a cost matrix $C$ and a perturbed version of it, $\tilde C = C + E$, where $E$ is an error matrix resulting from measurement noise. Denote by $\pi_\eps(C)$ the unique optimal transport plan corresponding to $C$ obtained by solving \eqref{eq: optimalPlan}.
  Then, 
  \[
    \lVert W_\eps^{(f)}(C) - W_\eps^{(f)}(C+E) \rVert_F \le \varepsilon^{-1} \lVert E\rVert_F.
  \]
  Further, if $E$ is a rank-one perturbation, i.e., $E= \eta \oplus \eta$, then,
  \[
    W_\eps^{(f)}(C) = W_\eps^{(f)}(C+\eta \oplus \eta).
  \]
\end{lemma}
  
\begin{proof}
The first statement follows from the fact that the set of hollow bistochastic matrices is a closed and convex set. Computation of $W_\eps^{(f)}(C)$ amounts to a projection of $-C/\varepsilon$ onto the set of hollow bistochastic matrices in Frobenius norm. The set of hollow bistochastic matrix can be rewritten as the intersection of two closed sets and is thus closed. It is a well known result that the metric projection onto a nonempty, closed convex set is contractive. 
For the second point, it is directly seen from the dual formulation \eqref{eq: Dual}, that if $u^*$ is optimal for $C$, then $u^*+\eta $ is optimal for $C+\eta \oplus \eta$.
\end{proof}
In contrast, conventional normalisation approaches such as the frequently used Markov normalisation \citep{van2018recovering, coifman2006diffusion} and symmetric normalisation \citep{belkin2003laplacian} do not in general have the capacity to remove this type of rank-1 noise.

\begin{remark}[Extension to RKHS]
  \label{rmk:rkhs}
    Looking at the form of \eqref{eq: Primal} and noting that $C_{ij} = \| x_i \|^2 / 2 + \| x_j \|^2 / 2 - \langle x_i, x_j \rangle$, we see that the problem is equivalent to projecting the kernel matrix $K_{ij} = \langle x_i, x_j \rangle$ onto the set of bistochastic matrices under the Frobenius norm. Other kernel matrices can be used, depending on what underlying RKHS one aims to use. 
    This choice can be made depending on the model one has in mind. The structure of the noise will play an important role.

    For the Gaussian kernel for instance, it is directly seen that in the same noise model as in Section~\ref{sec: SketchResults}, 
\begin{align*}
\exp\left(-\tfrac1\eps \|\tilde{x}_i -\tilde{x}_j \|^2  \right) 
&= \exp\left(-\tfrac1\eps \|{x}_i -{x}_j \|^2  \right) \exp\left(-\tfrac1\eps \langle x_i -x_j , \eta_i - \eta_j\rangle \right)\exp\left(-\tfrac1\eps \|\eta_i -\eta_j \|^2  \right) 
\\
 &=\exp\left(-\tfrac1\eps \|{x}_i -{x}_j \|^2  \right)\big( 1 + \Oh_p(m^{-1/2})\big)\exp\left(-\tfrac1\eps \|\eta_i -\eta_j \|^2  \right)
\\
 &=\exp\left(-\tfrac1\eps \|{x}_i -{x}_j \|^2  \right)\big( 1 + \Oh_p(m^{-1/2})\big), 
\end{align*}
provided that $\exp\left(-\frac1\eps \|\eta_i -\eta_j \|^2  \right)= 1 + \Oh_p(1)$. This last condition was not required in the case of the square Euclidean matrix. Using a Gaussian kernel induces multiplicative noise. 
Still, as we have just shown, if the one has the condition $\exp\left(-\frac1\eps \|\eta_i -\eta_j \|^2  \right)= 1 + \Oh_p(1)$, the fact that the methods performs a metric projection ensures some robustness.
\end{remark}

\subsubsection{QOT, nonparametric statistics and optimality}
The form of the quadratically regularised optimal transport plan in \eqref{eq: optimalPlan} is very much alike an Epanechnikov kernel, which is very often used in nonparametric statistics with the form 
\[
 u \mapsto \frac{\Gamma(2+d/2)}{\pi^{\frac{d}{2}}}(1 - u^\top u) \1_{\{ u^\top u \le 1\}}.
\]
The Epanechnikov kernel is often claimed to be the optimal non-negative kernel (in terms of asymptotic mean integrated square error) for the estimation of a twice differentiable density (although this statement has been debated \citep[Section~1.2.4]{tsybakov2008introduction}). Thus, the compactness of its support and the fact that the optimal dual potential is a function which is more adaptive to the data than a uniform bandwidth, may explain its favourable performance observed in the numerical examples we consider.

\subsection{Convergence of approximate discrete operator}
\label{sec: ConvOp}
We can now state our main theorem. Note that we consider functions defined on the ambient space $\reals^p$, as opposed to only on $\MC$, since in the manifold learning problem $\MC$ is unknown.  
\begin{theorem} 
\label{thm: MainConv}
Consider $g \in C^2(\reals^p)$, $g: \reals^p \to \reals$. Denote by $L_0$ and $Q_0$ the gradient and Hessian of $g$ at $x_0$, respectively. To simplify notation, set $X_0=x_0$. Take an i.i.d. sample $\{X_j\}_{j=1}^N$ from the uniform distribution on $\MC$ after embedding in $\reals^p$. Then, under Assumptions~\ref{assum: Rot}~and~\ref{assum: Assum2}, defining 
 \[
 \Delta^{OT} g(X_0) := \sum_{j=0}^N \overline{W}_{0,j}^{(f)} \big( g(X_0) -g(X_j)\big), 
 \]
 with $\overline{W}^{(f)}$  the approximate solution of the quadratically regularised OT problem $W_\eps^{(f)}(C)$ as introduced in Section~\ref{sec: SketchResults}, it holds that 
 \[
- 2 (N+1) K_{\eps, N}^{-1}  \  \Delta^{OT} g(X_0) \xrightarrow{L^2}  d L_0 \left\llbracket  0 ,
 \frac{\tilde{J}_{p,p-d}^\top \mathfrak{N}(x_0)}{2}
\right\rrbracket  +\frac12\tr \left[Q_0 \begin{pmatrix} 
 I_{d\times d} & 0 \\ 0 & 0
 \end{pmatrix} \right],
 \]
 provided that $K_{\eps, N} \to 0$ and $N/\eps \to 0$ when $N\to \infty$.
\end{theorem}

\subsection{Infinitesimal generator limit and spectral convergence}
A relatively general analysis of the convergence of graph Laplacians was carried out by \citet{ting2011analysis}, wherein consistency results are established for a general class of constructions leveraging connections to diffusion processes. We remark that when $\mathcal{M}$ is endowed with a uniform measure, a constant approximation of the potential is valid and so the operator resulting from quadratically regularised optimal transport falls under their framework \citep[Theorem 3]{ting2011analysis}. The assumptions are compatible with the ones that we make here, namely that $\mathcal{M}$ is a smooth, compact manifold, and the authors consider a general kernel of the form $K_N(x, y) = w_x^{(N)}(y) K_0\left( \frac{\| y - x \|}{h_N r^{(N)}_x(y)} \right)$.

In our setting where i.i.d. samples are drawn uniformly on $\mathcal{M}$, we invoke a constant potential approximation $u \sim \varepsilon^{\frac{2}{2 + d}} N^{\frac{-4}{2 + d}}$, we have (up to a multiplicative constant)
\begin{align*}
    K_N(x, y) &= \left[ \varepsilon^{\frac{2}{2 + d}} N^{\frac{-4}{d + 2}} - \| y - x \|^2 \right]_+ = \left[ 1 - \left( \dfrac{\| y - x \|}{ \varepsilon^{\frac{1}{2+d}} N^{\frac{-2}{d + 2}} }\right)^2\right]_+ = \varphi\left( \frac{\|y - x \|}{h^{(N)}}\right)
\end{align*}
Where the choice of kernel is the Epanechnikov kernel $\varphi(r) = (1 - r^2)_+$. The condition under which their theorem holds  is that $N h^{m+2}/\log N \to \infty$. In our case, this simplifies to $\varepsilon/(N \log N) \to \infty$, and this is compatible with the range of scalings $N^{1 - 2/d} \ll \varepsilon \ll N^2$ from our previous analysis.

\subsection{Similarity with the porous medium equation}
\label{sec:PME}
Consider again the continuous problem in equation~\eqref{eq: Primal}. Following \citet{lorenz2021quadratically} we write $\pi \in L^2(\MC)$ to be the density of a candidate transport plan w.r.t. product measure on $\MC \times \MC$, i.e., $\int \pi(x, y) \diff x \diff y = 1$. Then, the optimal transport plan $\pi^\star$ in the quadratically regularised problem must satisfy 
\[
    \frac{1}{\vM} \int_\MC \pi^\star(\iota(x_0), y) \diff y = \expec \left[ \pi^\star(\iota(x_0), X) \right] = \frac{1}{\vM^2}.
\]
Taking the relation $\pi^\star = \varepsilon^{-1} [u^\star \oplus u^\star - C]_+$ where $u^\star \in L^2(\MC)$ is the corresponding optimal dual potential, making the ansatz that $u^\star \sim \eps^\alpha$ and invoking \eqref{eq: B5} we have, for $\varepsilon \ll 1$,
\begin{align*}
\expec\left[ u^\star(\iota(x_0)) + u^\star(X) - \frac{1}{2} \| \iota(x_0) - X\|_2^2 \right]_+ &\sim \expec\left[ \left( \eps^\alpha - \lVert X - \iota(x_0)\rVert^2 \right) \1\{ \lVert X - \iota(x_0)\rVert\le \eps^{\alpha/2} \}  \right] \\
&\sim \frac{\lvert S^{d-1}\rvert}{d} \eps^{\tfrac{\alpha(d+1)}{2}}  + \Oh\left( \eps^{\tfrac{\alpha(d+2)}{2}}\right),
\end{align*}
(in the above multiplicative constants were dropped). The above quantity must behave asymptotically like $\varepsilon$ at leading order, and so matching exponents gives us
\[
\alpha+ \frac{d\alpha}{2} = 1 \Leftrightarrow \alpha = \frac{2}{2+d}.
\]
In the continuous case, the optimal potentials must thus behave in the first order like $\eps^\frac{2}{2+d}$.

There is a belief in the community that there should be some link between quadratically regularised optimal transport and a class of nonlinear partial differential equations known as the porous medium equation on $\reals^d$ for index $m = 2$ (see e.g., \citet{lavenant2018dynamical}), i.e., the equation
\[
\frac{\partial u}{\partial t }= \Delta(u^m), 
\]
where $u=u(x,t)$ and with an initial condition on $u$ at time $t=0.$ Starting from a Dirac mass of integral $\mathfrak{m}$ at the origin, the solution of the porous medium equation for $m =2$ is given by the Barenblatt-Prattle formula (\cite{vazquez2007porous}):
\[
u(x,t) = \max\left\{0, t^{-\frac{d}{2+d }}\left(\mathfrak{m} - \frac{ 1 }{4 (d+2)}\frac{ \lVert x\rVert^2 }{t^{\frac{2}{2+d}}}  \right)  \right\}. 
\]
or, rewriting terms, 
\[
u(x,t) = \max\left\{0, t^{-1}
\left(\mathfrak{m}\ t^{\frac{2}{2+d}}  - \frac{ 1 }{4 (d+2)} \lVert x\rVert^2   \right)  \right\}. 
\]
As the porous medium equation conserves mass, the integral of $u(x,t)$ over $\reals^d$ remains $\mathfrak{m}$. A key property of the porous medium equation which distinguishes it from the (linear) heat equation is that the solution remains compactly supported. As is evident from e.g., \eqref{eq: Dual}, this same property applies to the transport plans derived from quadratically regularised optimal transport. 

Perhaps closer to the theory of optimal transport, the porous medium equation of index $m$ can also be understood as the 2-Wasserstein gradient flow of the Tsallis entropy of order $m$ (see for example, the discussion in \citet{peyre2015entropic}). The Tsallis entropy generalises the Gibbs entropy: for $m = 1$, one recovers the Gibbs entropy, while for $m = 2$ it is corresponds to the squared $L_2$ norm of the density. It is remarkable that the squared $L_2$ norm is the functional that generates the porous medium equation as Wasserstein gradient flow, which is also the regularising functional used in quadratically regularised optimal transport exhibiting analogous sparsity and scaling behaviour. Furthermore in the entropy regularised setting where $m = 1$, optimal transport enjoys the celebrated connection to a theory of large deviations for Brownian motions and the Schr\"odinger problem (\cite{leonard2013survey}). One interesting theoretical question would be whether similar connections could hold in more general cases, e.g., $m > 1$. 

Although our work does not formally establish the existence of such a connection, it is interesting that the same types of exponents appear and that the solution of the porous medium equation is so close in form to the solution of the quadratically regularised optimal transport problem. Since the first version of this work, the understanding of the link between quadratically optimal transport and the porous medium equation has been improved, see \cite{garriz2024infinitesimal}.

\section{Results}
\label{sec: Perf}

\subsection{Application: manifold learning with heteroskedastic noise}\label{sec:ManifoldLearning}

We investigate application of bistochastic projections in Frobenius norm of the linear kernel (which from now on we refer to as the \emph{bistochastic L2} or \emph{QOT} kernel) in the setting of a one-dimensional manifold embedded in a high-dimensional ambient space, subjected to noise of varying intensity: this is the setting discussed in Section \ref{sec: RobandOpt}. Such a setting may occur for instance in single-cell RNA-sequencing data, where the underlying system may exhibit a continuous manifold structure \citep{moon2018manifold} and the noise level of cellular gene expression profiles may be determined by factors such as the number of reads sequenced for that cell \citep{landa2021doubly}.

We generate synthetic data shown in Figure \ref{fig:figure_spiral} comprised of $N = 10^3$ points, $\{ x_i \}_{i = 1}^N$ evenly spaced along the closed curve parameterised by 
\[
    \begin{bmatrix} x(t) \\ y(t) \\ z(t) \end{bmatrix} = \begin{bmatrix}\cos(t) (0.5\cos(6t)+1) \\ \sin(t) (0.4\cos(6t)+1) \\ 0.4\sin(6t)\end{bmatrix}, \quad t \in [0, 2\pi].
\]
We then randomly construct an orthogonal matrix $R \in \mathbb{R}^{d \times 3}$ for a chosen target dimension $d \gg 1$, i.e., $R^\top R = I$. Each point $x_i \in \mathbb{R}^3$ is then linearly embedded in the $d$-dimensional ambient space with additive noise via $\hat{x}_i = R x_i + \eta_i$: the noise model we use is 
\[
    \eta_i = \frac{Z_i}{\| Z_i \|} \rho(\theta), \quad 1 \leq i \leq N; \quad \rho(\theta) = 0.05 + 0.95\dfrac{1 + \cos(6\theta)}{2}, 
\]
where $Z_i$ follows the standard Gaussian distribution in $\mathbb{R}^{d}$. In order to eliminate the effect of scaling on the the spread parameter $\varepsilon$, we normalise our input data so that $N^{-2} \sum_{i, j = 1}^N \| \hat{x}_i - \hat{x}_j \|^2 = 1$. Given the noisy $d$-dimensional dataset $\{ \hat{x}_i \}_{i = 1}^N$, we construct affinity matrices using a suite of methods, and across a range of parameter values. For bistochastic constructions, we consider Frobenius projection of the linear kernel (QOT) and information projection of the Gaussian kernel (EOT), as well as Frobenius projection of a Gaussian kernel, which is the one used in \citep{zass2006doubly}. We also consider Epanechnikov and Gaussian kernels, which are related to the Frobenius and entropic bistochastic projections respectively, but do not satisfy the bistochastic constraint. Finally we also consider a standard $k$-NN construction, as well as the more sophisticated MAGIC method \citep{van2018recovering} which relies on $k$-NN graph construction as a first step. A summary of parameter choices is provided in Table \ref{tab:spiral_methods_parameters}. 

\begin{figure}[h!]
    \centering
    \includegraphics[width = \linewidth]{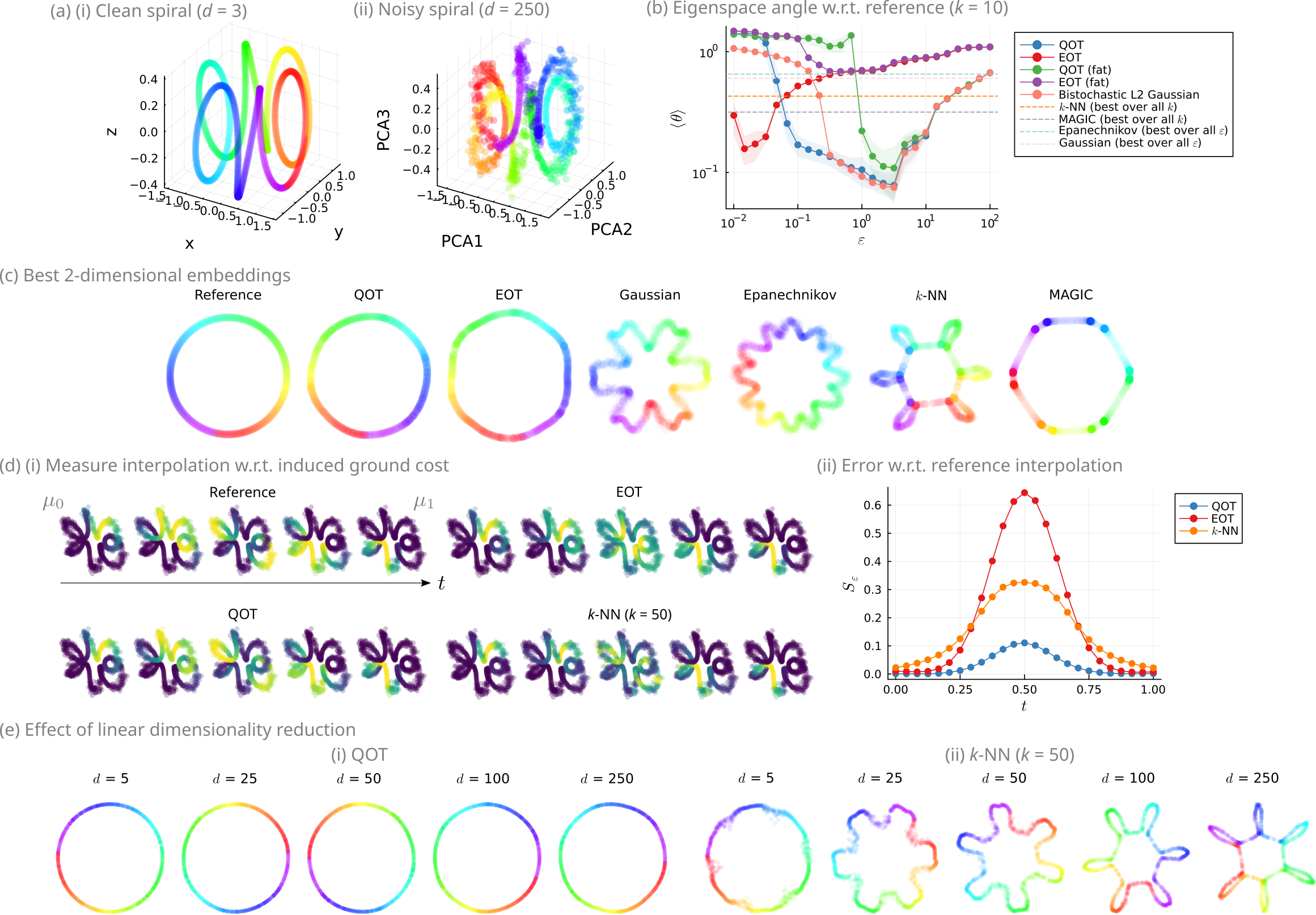}
    \caption{\textbf{Spiral in high dimension with non-uniform noise.} (a) (i) Clean points ($N = 1000$) sampled evenly from a closed spiral in 3 dimensions. (ii) Points embedded in $d = 250$ dimensions, subjected to non-uniform noise in the ambient high-dimensional space. (b) Laplacian eigenspace angles between a ground truth reference Laplacian and Laplacians obtained using various choices of affinity matrix construction and bandwidth parameter ($\varepsilon$) choices. (c) Best 2-dimensional Laplacian eigenfunction embeddings found for each affinity matrix construction. (d) (i) Interpolation between two measures $(\mu_0, \mu_1)$ supported on the spiral, computed as Sinkhorn barycenters using a ground cost induced by different affinity matrix constructions. (ii) Error (measured in terms of the Sinkhorn divergence between distributions) of measure interpolations with respect to a ground truth interpolation. (e) For fixed parameters, the spectral embedding obtained using the bistochastic Frobenius projection (QOT) and $k$-nearest neighbours ($k$-NN) with the top $d$ PCA coordinates as $d$ ranges from $5$ to $250$.}
    \label{fig:figure_spiral}
\end{figure}

In Figure \ref{fig:figure_spiral}(a), we show first in (i) the clean, low-dimensional data and in (ii) the noisy, high-dimensional data in $d = 250$ projected onto its first three principal components. We next investigate the spectral properties of different affinity matrix constructions from the noisy high-dimensional data. For an affinity matrix $W$ computed by one of the methods we consider, we first symmetrically normalise $W$ to form $\overline{W}$:
{
\begin{align*}
    \overline{W} = D^{-1/2} W D^{-1/2}, \quad D = \diag(W \1).
\end{align*}
}
The spectral decomposition of the corresponding (weighted, symmetrically normalised) Laplacian $\overline{L} = I - \overline{W}$ is then computed, yielding eigenvalues $\lambda_N \geq \lambda_{N-1} \geq \ldots \geq \lambda_{1} = 0$ and corresponding eigenvectors $v_N, \ldots, v_1$. The leading eigenvector $\lambda_1$ is trivial and so the $\ell$-dimensional \emph{eigenmap embedding} for $\ell = 1, \ldots, N-1$ is then the mapping 
\begin{align*}
    x_i \mapsto (v_{2}(x_i), v_{3}(x_i), \ldots, v_{\ell + 1}(x_i)).
\end{align*}
We construct a \emph{reference} graph from the clean 3-dimensional data $\{x_i\}_{i = 1}^N$ using $k$-nearest neighbours with $k = 3$ and compute from its Laplacian a set of reference eigenvectors $\{ v_i^\mathrm{ref} \}_{i = 2}^N$. Since the clean data are evenly spaced, this produces an embedding into a perfect circle in the first two non-trivial components $v_2^\mathrm{ref}, v_3^\mathrm{ref}$.

We reason that affinity matrix constructions that are best able to extract latent structure from the noisy, high-dimensional data should have leading Laplacian eigenspaces that are \emph{closer} to those of the reference Laplacian. To quantitatively measure agreement of eigenspaces, we compute the principal angle \citep{knyazev2002principal} between the eigenspaces spanned by the $k = 10$ leading eigenvectors: i.e., between $\operatorname{span}\{v_1, \ldots, v_{k}\}$ and $\operatorname{span}\{ v_1^\mathrm{ref}, \ldots, v_{k}^\mathrm{ref} \}$. Figure \ref{fig:figure_spiral}(b) summarises these results: it is clear that the graph constructions obtained using the Epanechnikov and Gaussian kernels perform quite poorly overall -- there are no parameter settings for which these methods perform well. This effect can be seen also in Figure \ref{fig:figure_spiral}(c), where the best eigenvector embeddings for Epanechnikov and Gaussian kernels exhibit clear high-frequency oscillations that arise from the effect of non-uniform noise. Methods based on $k$-NN graph constructions do better on the other hand, with MAGIC outperforming simple $k$-NN by some margin. The key distinction between these methods is that MAGIC has the capability of adaptively varying the effective bandwidth across regions of the dataset \cite{van2018recovering}, while $k$-NN uses a fixed unit edge weight. In Figure \ref{fig:figure_spiral}(c), artifacts reflecting the impact of noise are clearly visible for both methods, with $k$-NN more severely affected.

We find that bistochastic projections with respect to the Frobenius norm (QOT) and the relative entropy (EOT) significantly outperform all other affinity matrix constructions, in terms of eigenspace angle, and this can be visually confirmed in Figure \ref{fig:figure_spiral}(c) in which a near-perfect circular embedding is recovered. This agrees with the observations of Section \ref{sec: RobandOpt} and \citet{landa2021doubly} --  the process of bistochastic projection is able to remove the effect of the heteroskedastic noise when the dimensionality is sufficiently high. We find further that QOT improves upon the EOT in several respects: for a wide range of parameter values, the embedding obtained is closer in terms of the eigenspaces. More importantly, QOT yields good eigenspace agreement for a large range of $\varepsilon$ values, spanning over two orders of magnitude. On the other hand, EOT is much more sensitive to its parameter, yielding good results only within a narrow range of $\varepsilon$ values. We remark that since we plot the parameter $\varepsilon$ on a log-scale, our conclusions hold regardless of scale differences in $\varepsilon$ between methods.

To further illustrate the difference between the choice of bistochastic projection, we consider the problem of interpolating between two distributions $\mu_0, \mu_1 \in \PC(\XC)$ supported on the noisy spiral. For each affinity matrix construction, we consider the best parameter value found previously in terms of eigenspace angle. From an affinity matrix $\overline{W}$ we form the Laplacian $L = I - \overline{W}$ and its corresponding heat kernel $K = e^{-tL}$. It is not straightforward to directly compare different affinity matrix constructions since there is a priori no one-to-one correspondence between parameter choices, and so we choose $t$ for each affinity matrix construction so that the resulting heat kernel $K$ has a fixed mean \emph{perplexity} of 250. The perplexity of a probability distribution is
  \begin{equation}
    \operatorname{Perp}(p) = e^{-\sum_i p_i \log p_i}, \label{eq:perplexity}
  \end{equation}
  and can be understood as a measure of the effective neighbourhood size \citep{van2024snekhorn, van2008visualizing}. Indeed, for a uniform distribution $p$ supported on $k$ atoms one has $\operatorname{Perp}(p) = k$. A family of interpolating distributions is produced by solving a Sinkhorn barycenter problem \citep{peyre2019computational} between $(\mu_0, \mu_1)$ with varying weights $(1-s, s), 0 \leq s \leq 1$ using $K$ as the reference heat kernel \citep{solomon2015convolutional}. We show in Figure \ref{fig:figure_spiral}(d)(i) the interpolated distributions for the reference Laplacian (corresponding to the ground truth, noise-free data), both Frobenius (QOT) and information (EOT) bistochastic projections, as well as $k$-NN. We find that the interpolated distributions using the EOT-derived heat kernel shows evidence of ``short-circuiting'': this is likely due to the dense nature of the underlying weighted graph, allowing for connections between parts of the manifold that are far apart. On the other hand, the QOT-derived heat kernel yields interpolations that respect the underlying geometry and better resemble the reference interpolation. We quantify this in Figure \ref{fig:figure_spiral}(d)(ii) in terms of the Sinkhorn divergence \citep{feydy2019interpolating} between the interpolant and the ground truth as the interpolation coordinate $s$ varies. We find that using the QOT-derived heat kernel yields a much lower error compared to the other methods, providing further evidence that bistochastic projections in Frobenius norm are better suited to downstream applications by virtue of their sparsity.

In Figure \ref{fig:figure_spiral}(e) we consider the effect of linear dimensionality reduction via PCA, a pre-processing step that is standard when dealing with data of even moderately high dimension \citep{van2018recovering, schiebinger2019optimal, zass2006doubly, gorin2022rna}. We calculate PCA projections of the noisy data ranging from $d = 5$ to $d = 250$ (i.e., ambient) dimensions. For the optimal choice of parameter found previously, we show the Laplacian eigenvector embedding obtained from the QOT projection as the number of principal components varies. We do the same using $k$-NN for $k = 50$, and find that $k$-NN performance progressively degrades as the number of principal components is increased, and even when $d = 5$ the embedding exhibits irregularities. On the other hand, for a single fixed value of $\varepsilon$ the QOT projection achieves excellent performance across the full range of dimensionality reduction. This supports the results of Section \ref{sec: Rob}, that bistochastic kernel projections are robust to noise in high-dimensional settings. In this simple case PCA does not affect the results significantly, however in the case of an additive heteroskedastic noise model we note that preprocessing by PCA could impede the robustness as it modifies the contamination of pairwise distances.

Finally, we remark on the relevance of taking the bistochastic projection of the \emph{linear} kernel, and the necessity of the \emph{hollow} constraint. As mentioned in Section \ref{sec:generalized_projections}, the Gaussian kernel is the natural one (i.e., the one that corresponds to a local Euclidean geometry) for taking information projections, while the linear kernel is the natural one for Frobenius projections. In the context of the computational example we consider here, in Figure \ref{fig:figure_spiral}(b) we show the performance achieved for Frobenius projection of the linear kernel $-C/\varepsilon$ (QOT) and the Gaussian kernel $\exp(-C/\varepsilon)$ (Bistochastic L2 Gaussian), which we remind corresponds to first mapping into a RKHS. We observe that projecting the Gaussian kernel results in \emph{worse} performance than the linear kernel for $\varepsilon$ small, and for $\varepsilon$ large they begin to coincide, reflecting the fact that $\exp(-C/\varepsilon) \approx 1 - C/\varepsilon$ for $\varepsilon \gg 1$. For the hollowness constraint, comparing to the ``fattened'' case where a projection is sought among all symmetric bistochastic matrices rather than hollow ones, we find that the performance for both QOT and EOT projections is significantly worsened.

\subsection{Application: spectral clustering}

We consider spectral clustering \citep{von2007tutorial}, a problem which is closely related to -- but distinct from -- manifold learning. This problem was studied by \cite{zass2006doubly}, in which the setting considered was that of finding the bistochastic projection in Frobenius norm of an \emph{arbitrary} affinity matrix. On the other hand, our insight is that projection of the linear kernel (or equivalently, minus the matrix of pairwise distances) corresponds to the natural Euclidean space, and projecting a different kernel (e.g., the Gaussian kernel, as in \citep{zass2006doubly}) implicitly assumes a mapping first to a RKHS before taking the projection. Furthermore, as observed by \cite{landa2021doubly} it is important to project onto the set of \emph{hollow} bistochastic matrices in order to ensure the robustness property discussed in Section \ref{sec: RobandOpt}.  

We consider a mixture of $k = 3$ Gaussians, $\NC(\mu_i, \sigma_i), 1 \leq i \leq 3$ in dimension $d$ with $500$ points sampled from each mixture component. We choose 
\[
    \mu_1 = \begin{bmatrix} 0 \\ 1 \\ \0^{d-2} \end{bmatrix}, \:
    \mu_2 = \begin{bmatrix} \sin(2\pi/3) \\ \cos(2\pi/3) \\ \0^{d-2} \end{bmatrix}, \:
    \mu_3 = \begin{bmatrix} \sin(-2\pi/3) \\ \cos(-2\pi/3) \\ \0^{d-2} \end{bmatrix}.
\]
and $\sigma_1 = 0.3, \sigma_2 = 0.6, \sigma_3 = 1.0$. In Figure \ref{fig:gmm}(a) we show sampled points for $d = 50$ in PCA coordinates, coloured by the true labels. From these, we construct affinity matrices using QOT, EOT as well as $k$-NN. For QOT, we choose $\varepsilon = 1.0 \times \overline{C}$ where $\overline{C} = N^{-2} \sum_{ij} C_{ij}$ is the \emph{mean} of the squared Euclidean cost matrix. We recommend this heuristic as a general default, and we find it works well in benchmarking results. To ensure a fair comparison, the regularisation strength $\varepsilon$ for EOT is tuned until the mean perplexity matches that for QOT. The resulting perplexity distribution and affinity matrices are shown in Figure \ref{fig:gmm}(a)(ii-iii). Interestingly, while both methods produce a distribution over effective neighbourhood sizes, the EOT affinities show much more extreme variability than the QOT affinities. This is evident also from visual inspection of the EOT affinity matrix, where points within cluster 1 are much more tightly clustered than the others. 
The basis of spectral clustering is that cluster structure in data are reflected in the leading eigenvectors of the Laplacian matrix. To this end, we inspect the leading eigenvectors of Laplacians constructed using the QOT and EOT affinities in Figure \ref{fig:gmm}(b)(i-ii). It is immediately apparent that the first two non-trivial eigenvectors for the QOT affinities distinguish cleanly between clusters. On the other hand, for the EOT affinities effectively consist of only noise. 
Additionally, we contrast in Figure \ref{fig:gmm_supp_fig} the Laplacian spectra of QOT and EOT kernel matrices for $d = 10, 50, 250$. This highlights a stark difference between the two -- only the QOT projection exhibits clear gaps between the leading Laplacian eigenvalues, which is well known to be associated to presence of cluster structure \citep{von2007tutorial}.

For a systematic and unbiased assessment of performance, we benchmark QOT, EOT and $k$-NN constructions for $d = 10, 50, 250$ and varying hyperparameter values. We measure spectral clustering accuracy in terms of the normalised mutual information (NMI). Additionally, we propose a more direct measure for preservation of cluster information in the Laplacian spectrum. Introduce a \emph{template} basis encoded by $U \in \mathbb{R}^{N \times 3}$ where $U_{ij} = \ones_{y_i = j}$, where $y_i \in \{ 1, 2, 3\}$ is the ground truth cluster label of point $x_i$. In other words, $U$ corresponds to the Laplacian eigenbasis for an ``ideal'' affinity matrix with $k = 3$ disjoint components, each of which is a complete subgraph corresponding to a ground truth cluster. For a given affinity matrix constructed from data, we calculate the subspace angles between $U$ and the subspace spanned by the leading eigenvectors of the Laplacian. While in theory the leading $k$ eigenvectors should be sufficient to distinguish $k$ clusters, in practice more eigenvectors are often needed. Therefore we take the leading $2k = 6$ Laplacian eigenvectors. In Figure \ref{fig:gmm}(c) we find that the performance of $k$-NN rapidly worsens as the dimensionality $d$ increases, compared to QOT or EOT. On the other hand, the performance of QOT remains stable and outperforms EOT by a widening margin as $d$ increases. Our observations hold for both spectral clustering performance as measured by NMI as well as in terms of the subspace angle. Finally, in Figure \ref{fig:gmm}(d) we show performance with PCA dimensionality reduction as a preprocessing step, where the parameter value is chosen to be the optimal value found for $d = 250$. Similarly to the spiral example in Figure \ref{fig:figure_spiral}(e), we find that the performance of QOT remains virtually unchanged across the full range of $d_{PCA}$, while EOT appears to be unstable and $k$-NN rapidly deteriorates even for moderately large values of $d_{PCA}$.

\begin{figure}[h!]
  \centering
    \includegraphics[width = \linewidth]{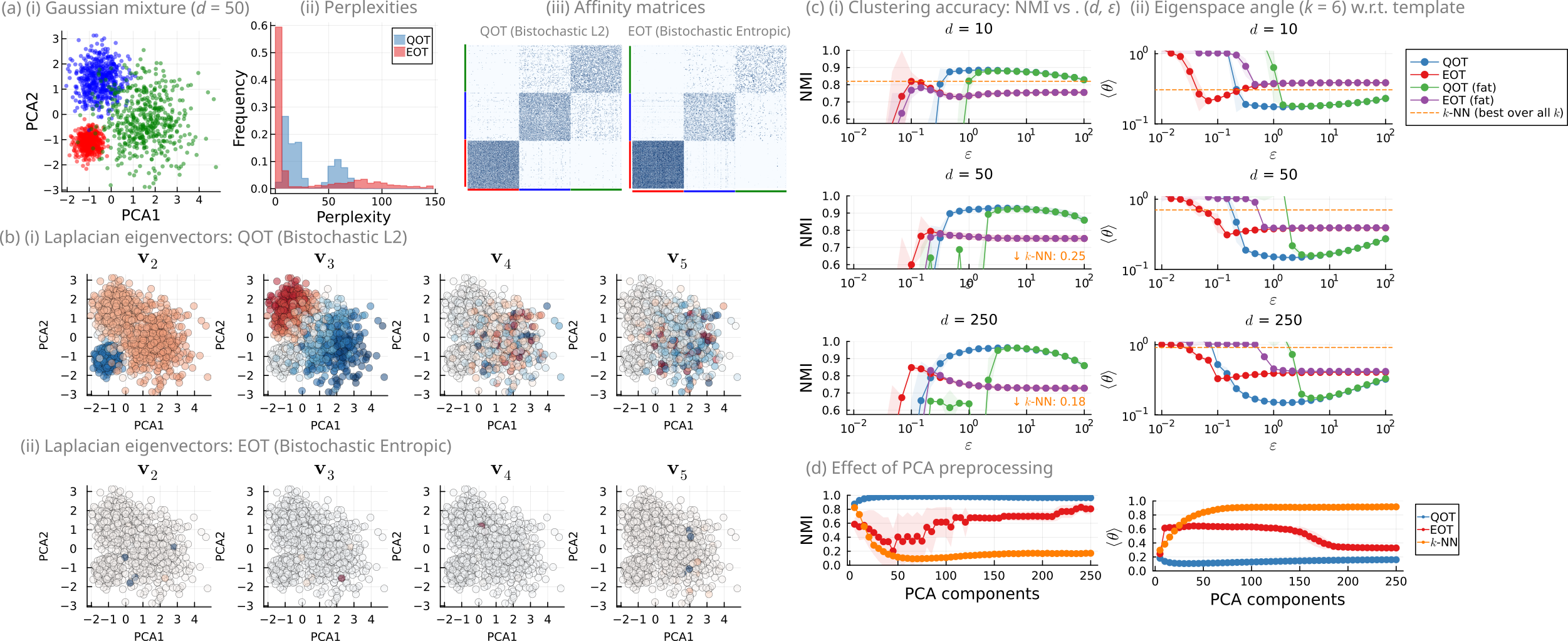}
    \caption{\textbf{Gaussian mixture model.} (a) (i) $N = 1500$ points sampled from a mixture of 3 Gaussians (500 points each) in $d = 50$ dimensions, shown in PCA coordinates and coloured by their true labels. (ii) Distribution of effective neighbourhood sizes (perplexities) for QOT (Frobenius) and EOT (entropic) bistochastic projections, where the mean perplexity is 30 in both cases. (iii) Corresponding affinity matrices. (b) (i-ii) Leading eigenvectors of the graph Laplacian constructed from each affinity matrix. (c) Performance for varying $(d, \varepsilon)$ measured in terms of (i) normalised mutual information (NMI) of spectral clustering result and (ii) eigenspace angle of leading 6 eigenvectors to template subspace. (d) Performance for different numbers of principal components with fixed parameters $\varepsilon$, $k$.}
  \label{fig:gmm}
\end{figure}

\begin{figure}[h!]
  \centering
    \includegraphics[width = \linewidth]{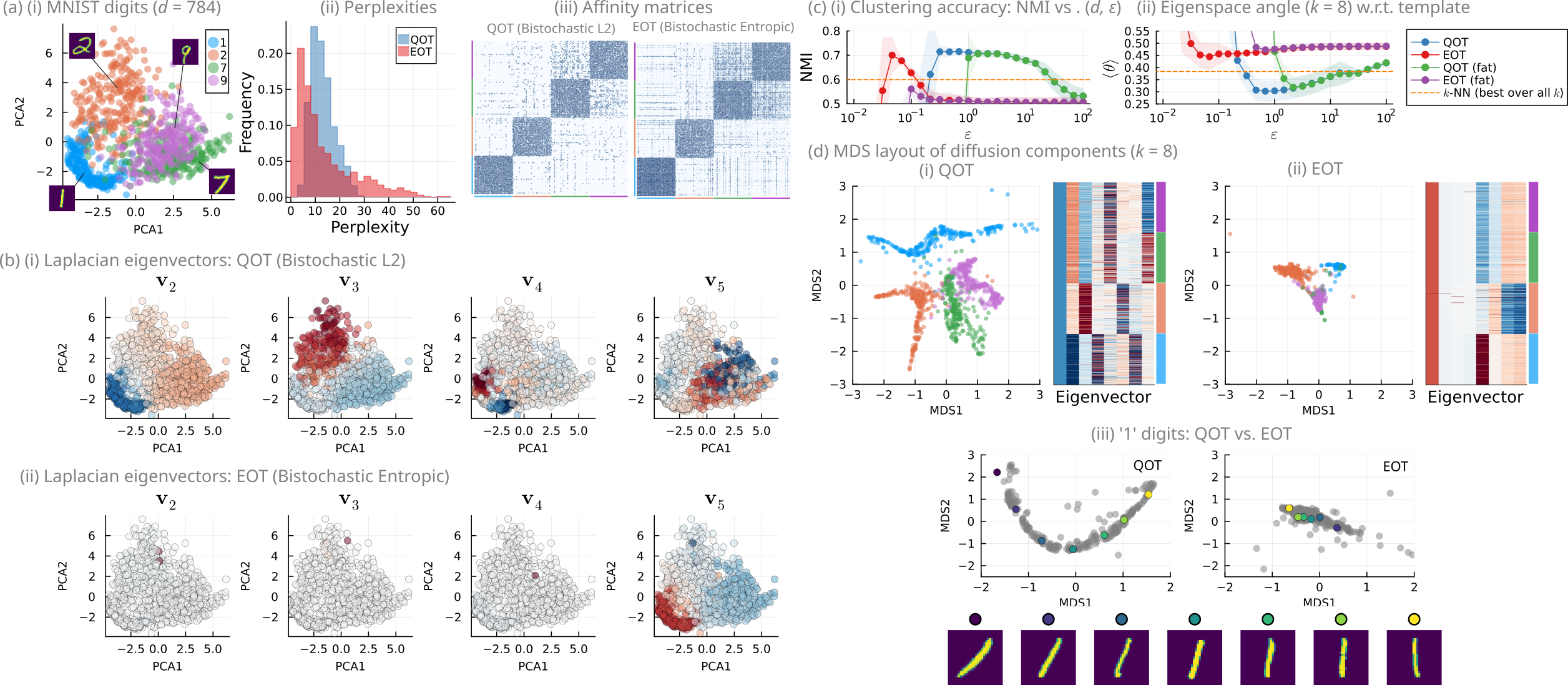}
    \caption{\textbf{MNIST dataset.} (a) (i) $N = 750$ MNIST sampled from \textsf{1}, \textsf{2}, \textsf{7}, \textsf{9} (250 images each) in $d = 784$ dimensions, shown in PCA coordinates and coloured by their true labels. (ii) Same as in Figure \ref{fig:gmm}(a)(ii) but with mean perplexity 13. (iii) Corresponding affinity matrices. (b) Same as in Figure \ref{fig:gmm}(b). (c) Performance measured in terms of (i) normalised mutual information (NMI) of spectral clustering result and (ii) eigenspace angle of leading 8 eigenvectors to template subspace. (d) (i-ii) MDS layouts of leading 8 diffusion components and corresponding eigenvectors, for L2 and entropic bistochastic projections. (iii) MDS layout of diffusion components for \textsf{1} digits recovers a 1-dimensional manifold corresponding to a rotation for bistochastic L2 projection, while this is corrupted in the entropic case.}
  \label{fig:mnist}
\end{figure}

\subsection{Application: MNIST handwritten digits}
  
We consider the MNIST image dataset of handwritten digits as an application to real world data. We consider the digits \textsf{1, 2, 7, 9} and sample 250 images each, constructing a dataset of $N = 750$ points in dimension $d = 784$. We illustrate these in PCA coordinates in Figure \ref{fig:mnist}(a)(i), coloured by the ground truth labels. For QOT, we choose $\varepsilon = 1.0 \times \overline{C}$ as previously, and for EOT we tune its regularisation strength to match the perplexity of QOT. In Figure \ref{fig:mnist}(a)(ii-iii) we show the distribution of perplexities as well as affinity matrices for both QOT and EOT affinities. As in the case of Gaussians (Figure \ref{fig:gmm}(a)), we find that there is more variability in the perplexities for the EOT affinities. In Figure \ref{fig:mnist}(b)(i-ii) we show leading Laplacian eigenvectors for the QOT and EOT affinities. As before in the case of QOT affinities, we find that the eigenvectors highlight underlying features in the data -- $v_2, v_3$ distinguish the \textsf{1}s and \textsf{2}s, while $v_4$ identifies variability among the \textsf{1}s. In contrast, we find that the leading EOT eigenvectors $v_2$-$v_4$ consist predominantly of noise, with $v_5$ distinguishing the \textsf{1}s.

Carrying out the same benchmarking as in the Gaussian case but this time for $k = 4$ clusters, we find that both QOT and EOT perform similarly and outperform $k$-NN by a significant margin in terms of NMI (Figure \ref{fig:mnist}(c)(i)). However, EOT yields good performance only for a narrow range of $\varepsilon$, while QOT is stable across a much wider range of parameter values. In terms of eigenspace angle relative to a ground truth cluster template (Figure \ref{fig:mnist}(c)(ii)), we find that QOT outperforms both $k$-NN and EOT by a significant margin. The observation on the eigenspace angles suggests that the QOT affinities better encode cluster information in their eigenspaces compared to EOT affinities, but this information may not all be used by the downstream spectral clustering algorithm, as indicated by similar NMI scores for the two methods. Figure \ref{fig:mnist}(d) provides evidence that affirms this conclusion: calculating diffusion components from each affinity matrix, we show a 2-dimensional MDS layout of the diffusion embeddings alongside the leading $2k = 8$ Laplacian eigenvectors. Visually, with the QOT affinities a clear separation between \textsf{7}s and \textsf{9}s is achieved. On the other hand, for the EOT affinities they are superimposed. The cause of this is immediately apparent from inspection of the eigenvectors: for QOT, $v_8$ clearly separates \textsf{7}s and \textsf{9}s, however by default spectral clustering only utilises the leading $k = 4$ Laplacian eigenvectors. By comparison, several of the leading EOT eigenvectors are corrupted with noise, and no distinguishing information between \textsf{7}s and \textsf{9}s is apparent from the spectral information under consideration. 

Finally, we again showcase the suitability of the QOT kernel for \emph{manifold} learning in this setting. The MNIST \textsf{1}s approximately lie on a one-dimensional manifold that describes the slantedness of the digits \citep{schmitz2018wasserstein, seguy2015principal}. In Figure \ref{fig:mnist}(d)(iii) we show MDS layouts of diffusion components on the subset of \textsf{1}s, for both QOT and EOT affinities. We find that the QOT affinities produce a clear 1-dimensional embedded structure along which the progressive change in angle is apparent. In contrast, the EOT affinities produce a much noisier embedding from which this trend is not as readily observed. 

\subsection{Application: single cell RNA sequencing data}

Datasets of high-throughput measurements made from biological systems are naturally high dimensional and usually subject to significant noise, both biological and technical. Single cell RNA sequencing datasets typically comprise of on the order of $10^3 - 10^6$ cells and have an ambient dimensionality on the order of $10^4$, or the number of expressed genes in an organism. In many cases, the measured cells are captured along a continuous biological process such as development. The high-dimensional cell state data can be thought of as an embedded lower dimensional manifold whose intrinsic geometry captures biological details \cite{moon2018manifold, van2018recovering}. 

We consider one such dataset, originally published in \citet{la2018rna} that captures the spectrum of cell states across the development of mouse hippocampus. Starting from a sample of 5,000 cells we apply the standard filtering and pre-processing pipeline (filter out low expression cells and genes, normalise reads per cell, $\log(x+1)$ transformation, filter for highly variable genes and standardisation \citep{luecken2019current}). This leads to a $5000 \times 2239$ matrix of normalised cell-by-gene expression values. We computed affinity matrices using QOT and EOT, as well as $k$-NN with unit edge weights. For QOT we use $\varepsilon = 1.0 \times \overline{C}$, and we tune the regularisation strength of EOT to match the perplexity of QOT. For $k$-NN, we chose $k$ to be equal to the mean perplexity of QOT. For each affinity matrix, we transformed the cells into a 10-dimensional spectral embedding. In the spectral embedding coordinates we calculated the squared distance of each cell to an origin cell which was selected based on origin-cell probabilities calculated by \citet[Figure 3]{la2018rna} from RNA velocity reconstructions. 

We show in Figure \ref{fig:single_cell} on $t$-SNE coordinates the distance-to-origin cell as measured in the Laplacian eigenmap embeddings calculated based on the QOT and EOT affinity matrices, as well as the $k$-NN adjacency matrix. We observe that the QOT kernel produces an embedding in which the ends of branches of developmental lineages are distant fro the origin cell. This is in agreement with the end-state predictions calculated by \citet[Figure 3]{la2018rna} from RNA velocity reconstructions. On the other hand, the EOT and $k$-NN kernel constructions miss certain lineages.

Despite being a staple of single cell data analysis pipelines \citep{moon2018manifold, la2018rna}, the use of $k$-nearest neighbours for constructing affinity matrices has rarely been questioned \citep{gorin2022rna}. Our findings suggest that construction of affinity matrices by $k$-NN may be suboptimal for high-dimensional data, a regime in which we find the QOT construction to perform favourably.

\begin{figure}[h!]
    \centering
    \begin{subfigure}{\linewidth}
        \includegraphics[width=\linewidth]{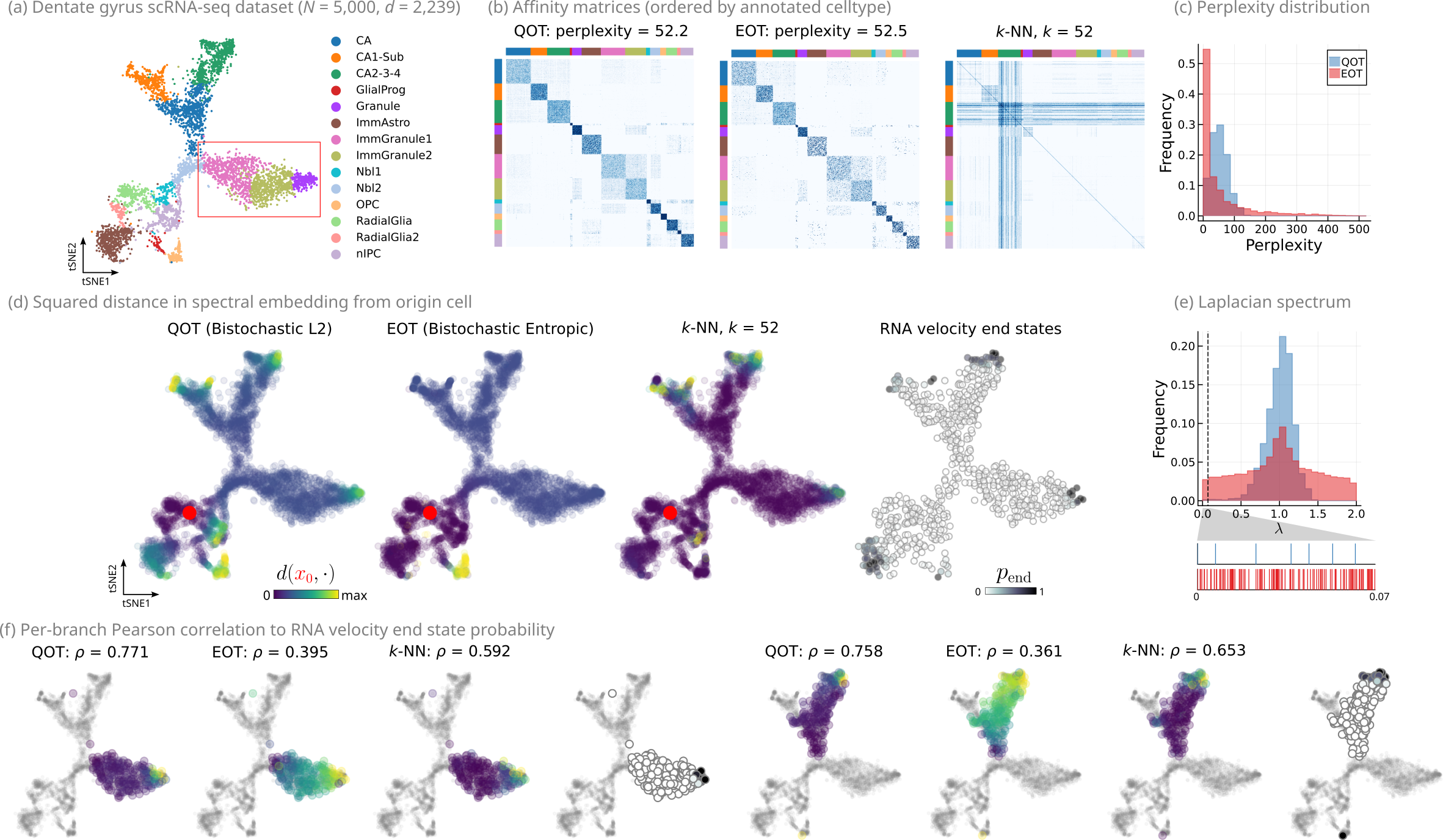}
    \end{subfigure}
    \caption{\textbf{Single cell RNA sequencing dataset.} (a) Sub-sampled dentate gyrus dataset \citep{la2018rna} (5000 cells, 2239 dimensions) visualised in $t$-SNE coordinates from the original publication, and coloured by celltype annotation. (b) L2 and entropic bistochastic projections of the linear kernel, ordered by celltype annotation. (c) Distribution of effective neighbourhood sizes (perplexities) for L2 and entropic bistochastic projections. (d) Left: Cells coloured by squared distance from origin cell (red) measured in terms of the 10-dimensional spectral embedding, for L2 and entropic bistochastic projections as well as $k$-NN affinity matrix constructions. Right: terminal cell states predicted by \citep{la2018rna} using RNA velocity estimates. (e) Distribution of Laplacian eigenvalues for L2 and entropic bistochastic projection. Inset shows occurence of leading eigenvalues in the range $[0, 0.07]$. (f) Same as (d) but shown for \textsf{Granule} and \textsf{CA} branches, together with the Pearson correlation between spectral embedding distance to origin and RNA velocity end state probability.}
    \label{fig:single_cell}
\end{figure}

\begin{figure}[h!]
    \centering
    \begin{subfigure}{\linewidth}
        \includegraphics[width=\linewidth]{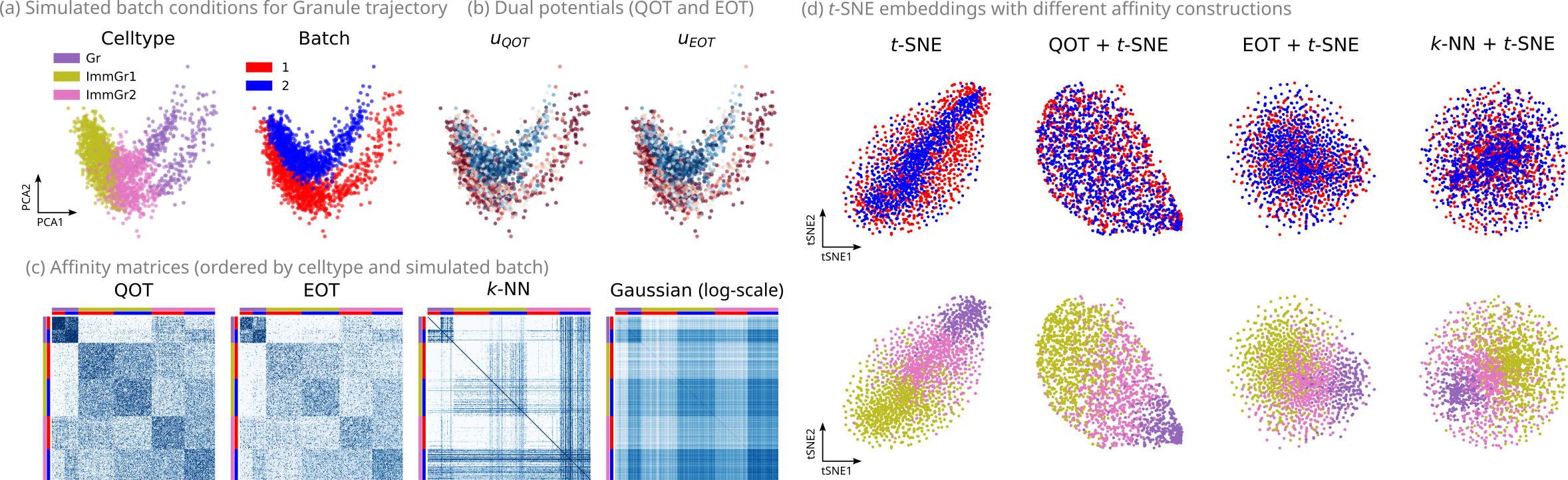}
    \end{subfigure}
    \caption{\textbf{Simulated scRNA-seq batch effects.} (a) 2000 cells sampled from Granule trajectory with simulated batch effect. (b) Dual potentials from QOT and EOT. (c) Affinity matrices for various constructions, coloured by celltype and batch. (d) $t$-SNE embeddings obtained using different affinity constructions. }
    \label{fig:batcheffect}
\end{figure}

As a further demonstration of the general-purpose potential of QOT affinity matrices manifold learning and dimensionality reduction, we focus on the sub-trajectory consisting of ImmGranule(1, 2) and Granule cells (see box in Figure \ref{fig:single_cell}(a)). In particular, to illustrate the robustness to heteroskedastic noise outlined in Sections \ref{sec: SketchResults} and \ref{sec: RobandOpt}, we simulate batch effects similarly to the approach described in \citep[Section 3.2]{landa2021doubly}. Starting with 2000 cells randomly sampled from the trajectory, we form a $2000 \times 2239$ cell-by-gene matrix $X_{ij}$ and row-normalise so that each cell $X_i$ corresponds to a probability vector $p_i$ over genes. A simulated gene expression profile for cell $i$ with batch effect is then obtained as a sample from $\mathrm{Multinomial}(p_i, n_i)$, where $n_i$ is the total number of reads assigned to cell $i$. In summary, for each cell $X_i$ a simulated expression profile $\tilde{X}_i$ is obtained following 
  \[
    \tilde{X}_i \sim \mathrm{Multinomial}(p_i, n_i), \quad p_{ij} = \frac{X_{ij}}{\sum_{j} X_{ij}}, 
  \]
  where $n_i = 100$ for 1000 randomly selected cells and $n_i = 250$ for the rest. Importantly, \citet{landa2021doubly} show in Section 3.2.3 of their paper that this noise model satisfies the requirements for the observation model for robustness to heteroskedastic noise (Proposition \ref{prop: RobHetNoise}). The same pre-processing steps as before are then applied to the simulated gene expression matrix $\tilde{X}$.

  We show in Figure \ref{fig:batcheffect}(a) the simulated data in PCA coordinates, from which presence of a batch effect is apparent. We then build affinity matrices using QOT and EOT constructions. As before, we tuned the parameters for each method to achieve a mean perplexity of 50. We illustrate in Figure \ref{fig:batcheffect}(b) the optimal potentials $u_{QOT}, u_{EOT}$ obtained from each method. As we expect from the discussion in \ref{sec: SketchResults}, the batches correspond to different noise levels and we find that this is empirically reflected in the values of the dual potential for both bistochastic projections. In Figure \ref{fig:batcheffect}(c) we show the affinity matrices obtained from QOT and EOT, as well as $k$-NN (with $k = 50$) and a Gaussian kernel with the same bandwidth as EOT (i.e., $\exp(-r^2 / 2\sigma^2)$ with $\varepsilon = 2\sigma^2$). For ease of visualisation, we show the Gaussian kernel on a log-scale. From these visualisations it becomes clear that both QOT and EOT are able to uncover structure corresponding to celltypes, although signal arising from the batches are still apparent. On the other hand, $k$-NN and the Gaussian kernel exhibit structure that is largely corresponds to variation in read number in the simulated batches. Finally, as an illustration of the practical consequences of the choice of affinity matrix, we produce $t$-SNE embeddings using different affinity matrices and show them in Figure \ref{fig:batcheffect}(d). Specifically, we employ default $t$-SNE (Gaussian affinities from 50 PCs and perplexity 30), as well as $t$-SNE using QOT, EOT and $k$-NN affinities (using openTSNE implementation \citep{polivcar2024opentsne}). Variation due to the presence of batches is clearly present in the case of default $t$-SNE, and using EOT and $k$-NN affinities produces embeddings with artifacts which do not clearly reveal the presence of an underlying trajectory. By contrast, $t$-SNE with QOT affinities is robust to batch effects (evidenced by the overlap between the two batches) and reveals the underlying trajectory.

\section{Conclusion and open questions}
\label{sec: Conclusion}

Our work opens the door for further applications and analysis of the bistochastic projection of affinity matrices in the Frobenius norm. Understanding finely the relationship between the porous medium equation and regularised optimal transport one concrete interesting theoretical problem; we only strengthen the apparent parallel in Section~\ref{sec:PME}. The understanding might further help explain why the methods perform well and more precisely help identify the weaknesses. Furthermore, determining the limiting operators that arise from this scheme for other settings is still an open question -- could one derive finer results on the potentials or further exploit the metric projection property?
Even though first results are given in that direction in Section~\ref{sec: RobandOpt},   providing a comprehensive, quantitative rationale for the over-performance observed in Section~\ref{sec: Perf} would be welcome. A finer description of the finite sample case would be valuable. Also, understanding when the method is not suitable would also be important for a controlled, widespread use.

The recent paper of \citet{van2024snekhorn} studies a variant of bistochastic scalings for affinity matrices, where a row-wise constraint on the effective neighbourhood size, or perplexity, is imposed. This is in contrast to the settings studied in the present paper as well as \citep{landa2021doubly, zass2006doubly}, where a global regularisation is imposed and the resulting effective neighbourhood size may vary across the dataset. While in this paper we focus on the latter and its corresponding theoretical properties, an interesting question for future work would be whether similar row-wise constraints can be introduced for Frobenius bistochastic projections, and if so, whether the induced sparsity improves performance for downstream tasks.

\acks{
S. Zhang acknowledges support from an Australian Government Research Training Program Scholarship and an Elizabeth and Vernon Puzey Scholarship. G. Mordant acknowledges the support of the DFG within CRC1456-A04. G. Schiebinger was supported by a MSHR Scholar Award, a CASI from the Burroughs Wellcome Fund and a CIHR Project Grant.}

\newpage

\appendix

\section{Supplementary theory}
\subsection{General regularised optimal transport}
\label{sec:general_OT}

For simplicity we will restrict our focus to discrete measures. 

\begin{definition}[Regularized optimal transport]
Consider a two discrete probability measures $\alpha, \beta$ supported on $n$ and $m$ points respectively. Recall that $\Pi(\alpha, \beta)$ denotes the set of couplings between $\alpha$ and $\beta$. Consider additionally  a smooth convex function $\Omega : \reals^{n\times m} \to \reals$ and a cost matrix $C \in \reals^{n\times m}$. Then, for $\eps >0$, the optimisation problem 
\[
        \min_{\pi \in \Pi(\alpha, \beta)} \langle C, \pi \rangle + \varepsilon \Omega(\pi)   
\]
   is called a \emph{regularised optimal  optimal transport problem}.
\end{definition}

\begin{lemma}[Dual of regularized optimal transport]
    Consider the regularised optimal transport problem with smooth, convex regulariser $\Omega$.
    \[
        \min_{\pi \in \Pi(\alpha, \beta)} \langle C, \pi \rangle + \varepsilon \Omega(\pi)   
    \]
    Then, the corresponding dual problem is 
    \[ 
        \sup_{u, v} -\langle u, \alpha \rangle - \langle v, \beta \rangle - \varepsilon \inf_{\rho \ge 0} \Omega^\star\left( - \frac{C + u \oplus v - \rho}{\varepsilon}\right).
      \]
    where $\Omega^\star$ is the Legendre transform of $\Omega$ in its argument. 
\end{lemma}
In the setting of optimal transport with a quadratic regularisation, we choose 
\[
\Omega(\pi) = \frac{1}{2} \| \pi \|_{\mu \otimes \nu}^2 = \frac{1}{2} \left\langle \frac{d\pi}{d\mu \otimes d\nu}, \pi \right\rangle.
\]
A straightforward computation reveals that 
\[
    \Omega^\star(\rho) = \frac{1}{2} \| \rho (\mu \otimes \nu) \|_{\mu \otimes \nu}^2. 
\]
It follows that the dual for the quadratically regularised optimal transport problem is 
\[
    \sup_{u, v} \langle u, \alpha \rangle + \langle v, \beta \rangle - \frac{1}{2\varepsilon} \| [u \oplus v - C]_+ (\mu \otimes \nu) \|_{\mu \otimes \nu}^2. 
\]
\begin{lemma}[Dual of hollow regularized optimal transport]
\label{eq: DualHollow}
    Consider the hollow regularised optimal transport problem with smooth, convex regulariser $\Omega$, denoting by $\hol$ the set of matrices with zero diagonal:
    \[
        \inf_{\pi \in \Pi(\alpha, \beta) \cap \hol} \langle C, \pi \rangle + \varepsilon \Omega(\pi)   
    \]
    Then, the corresponding dual problem is 
    \[ 
        \sup_{u, v,D} -\langle u, \alpha \rangle - \langle v, \beta \rangle - \varepsilon \inf_{\rho \ge 0} \Omega^\star\left( - \frac{ C+D + u \oplus v - \rho}{\varepsilon}\right),
    \]
    where D is in the set of diagonal matrices.
\end{lemma}
\begin{proof}
The Lagrange formulation of the problem reads
\[
   \inf_{\pi\geq 0} \sup_{u\in \reals^n ,v\in \reals^m,D} \langle C, \pi\rangle + \varepsilon \Omega(\pi) + \langle \alpha - \pi \1, u\rangle  + \langle \beta - \pi^\top \1, v\rangle + \langle D, \pi\rangle,
\]
where the optimisation for $D$ occurs on the set of diagonal matrices. On thus gets the classical dual formulation for the modified cost matrix.
\end{proof}

\begin{remark}[A note on the case of empirical measures]
    In the discrete setting where $\mu = N^{-1} \sum_i \delta_{x_i}, \nu = M^{-1} \sum_j \delta_{y_j}$, one has $\frac{\varepsilon}{2} \| \pi \|_{\mu \otimes \nu}^2 = \frac{\varepsilon N^2}{2} \| \pi \|_{\1\otimes\1}^2$. 
\end{remark}

\subsection{Details on generalised bistochastic information projections}

For a parameter $q \in [0, 1)$, a deformed analogue of the entropy, the $q$-entropy, is defined by \citet{bao2022sparse} as $\phi(x) = \frac{1}{2-q}( x \log_q x - x), x \ge 0$
where the $q$-analogues of the exponential and its inverse are defined for $q \in [0, 1)$ as
\begin{align*}
    \exp_q(x) &= [1 + (1-q) x]_+^{1/(1-q)},  \\ 
    \log_q(x) &= (1-q)^{-1} (x^{1-q} - 1).   
\end{align*}
Since $\phi$ is differentiable and convex on its domain, it generates a Bregman divergence on matrices with the formula 
\[
    \Phi(A | B) = \frac{1}{2-q} \sum_{i, j} \left[ A_{ij} \log_q A_{ij} - A_{ij} \log_q B_{ij} - A_{ij} B_{ij}^{1-q} + B_{ij}^{2 - q}\right].
\]
It is straightforward to verify that as $q \to 1$ one recovers the standard exponential and logarithm, and thus $\Phi$ becomes the KL divergence. When $q=0$, corresponding to the squared Frobenius norm, we have that $\phi(x) = x^2/2 - x + \iota(x \ge 0)$, $\phi^\star(u) = \frac{1}{2} [ 1 + u]_+^2, \nabla \phi^\star(u) = [1 + u]_+$ and so
\[
    \Phi(A | B) = \frac{1}{2} \| A_{ij} - B_{ij} \|_2^2. 
\]

Consider then the generalised Bregman projection problem
\begin{align}
    \argmin_{A \in \Pi(\1, \1)} \varepsilon \Phi^\star\left( \frac{-C}{\varepsilon}, \nabla \phi(A) \right) \label{eq:bregman_dual_projection}
\end{align}
where $\Phi^\star$ is the Bregman divergence generated by $\phi^\star$, the convex conjugate of $\phi$. When $q \to 1$, $\phi^\star(u) \to e^u$ and is smooth on its domain, so the objective function is exactly equal to $\varepsilon \Phi(A, \nabla \phi^\star(-C/\varepsilon))$. For other cases, however, $\phi^\star$ fails to be smooth and we must consider the problem \eqref{eq:bregman_dual_projection} to avoid loss of information. Expanding the definition of $\Phi$, 
\begin{align*}
    \varepsilon \Phi^\star\left( \frac{-C}{\varepsilon}, \nabla \phi(A) \right) &= \varepsilon \sum_{ij} \phi^\star(-C_{ij}/\varepsilon) - \phi^\star(\nabla \phi(A_{ij})) - \nabla \phi^\star (\nabla \phi(A_{ij})) (-C_{ij}/\varepsilon - \nabla \phi(A_{ij})) \\
    &= \varepsilon \sum_{ij} \phi^\star(-C_{ij}/\varepsilon) - \phi^\star(\nabla \phi(A_{ij})) \\
    &\quad+ \nabla \phi^\star(\nabla \phi(A_{ij})) C_{ij}/\varepsilon + \nabla \phi^\star (\nabla \phi(A_{ij})) \nabla \phi(A_{ij}) \\
    &= \varepsilon \sum_{ij} \phi^\star(-C_{ij}/\varepsilon) + A_{ij} C_{ij}/\varepsilon + \phi(A_{ij}),
\end{align*}
Which is, up to a constant, exactly the regularised optimal transport problem. Since $\phi$ is convex, by the above expansion we see that the problem in $A$ remains convex. In the above, since $A_{ij} \ge 0$, we have that $\nabla \phi \circ \nabla \phi^{-1} = \operatorname{id}$, and also $\phi^\star (\nabla \phi(x)) = x \nabla \phi(x) - \phi(x)$.

We remark that when $q = 1$, $\exp(-C/\varepsilon)$ is the well-known Gibbs kernel and for $q \in [0, 1)$ one may write $\exp_q(-C/\varepsilon)$ to be the $q$-Gibbs kernel \citep{bao2022sparse}. However, $\exp_q$ involves a thresholding of the cost matrix and thus is not equivalent to the regularised optimal transport problem. (This is a result of $\phi^\star$ failing to be smooth in such cases.) So, while the general problem is \enquote{formally} the projection of the $q$-Gibbs kernel, it is more precisely understood as being done via the dual formulation.

\subsection{Additional lemmas and proofs}
\label{sec: AppLem}
In this appendix, we collect some useful lemmas. 
\begin{proof}(Proof of Proposition~\ref{prop: RobHetNoise})
First we establish that projection in the set of hollow bistochastic matrices is invariant with respect to addition of diagonal matrices. Indeed, for a diagonal matrix $K$;
\[
 \min_{\pi \in \KCo} \sum_{i,j} \pi_{i,j} \tilde{c}_{i,j} + \frac{\eps}{2} \| \pi \|_F^2 =  \min_{\pi \in \KCo} \sum_{i,j} \pi_{i,j} ( \tilde{c}_{i,j}+K_{i,j}) + \frac{\eps}{2} \| \pi \|_F^2, 
\]
as $\pi_{i,i}= 0$. Then, in the error model that we have,
\[
\| \tilde{x_i} - \tilde{x_j}\|^2 = \| {x_i} - {x_j}\|^2 + 2\langle \tilde{x_i} - \tilde{x_j} ,\eta_i - \eta_j\rangle +  \|\eta_i\| ^2 +  \|\eta_j\| ^2 - \langle \eta_i , \eta_j\rangle.
\]
Note that, for $i\neq j$, by the same arguments as in \citet[Lemma~4]{landa2021doubly}
\[
 2\langle \tilde{x_i} - \tilde{x_j} ,\eta_i - \eta_j\rangle - \langle \eta_i , \eta_j\rangle = \mathcal{O}_p(m^{-1/2}).
\]
However, when $i=j$
\[
\| \tilde{x_i} - \tilde{x_j}\|^2 = \| {x_i} - {x_j}\|^2 =0.
\]
One thus has that the noisy cost matrix is the pure cost matrix plus a rank-one perturbation plus a noise matrix $E$ up to adding the lacking diagonal element of the rank-one perturbation.
We can thus add the diagonal missing part so that 
\[
 \min_{\pi \in \KCo} \sum_{i,j} \pi_{i,j} \tilde{c}_{i,j} + \frac{\eps}{2} \| \pi \|_F^2 =  \min_{\pi \in \KCo} \sum_{i,j} \pi_{i,j} ( {c}_{i,j}+\eta_i +\eta_j + E_{i,j} ) + \frac{\eps}{2} \| \pi \|_F^2, 
\]
where $E_{i,j}$ is hollow and $\Oh_p (m^{-1/2})$. 
Finally observe that, 
\[
\sum_{i,j} \pi_{i,j} \eta_i 
\]
is a constant and independent of the optimal permutation chosen.
Therefore, one gets that $W_\eps^{(f)}(\tilde{C})= W_\eps^{(f)}(C+E)$ and one can use Lemma~\ref{lem: MetProj}
to deduce the claim.
\end{proof}

\begin{lemma}
\label{lem: B5}
For $x_0 \in \MC$,  it holds for $r$ sufficiently small that  
\begin{align}
\label{eq: B5}
\nonumber
\expec\left[ f(X; r) \1\{ \lVert X - \iota(x_0)\rVert\le r \}  \right] &= \frac{\lvert S^{d-1}\rvert}{d \vM} f(\iota(x_0); r)   r^d  \\ \nonumber 
&\qquad +
\frac{\lvert S^{d-1}\rvert}{d(d+2)}\Big(
\frac{1}{2\vM} \Delta f(\iota(x_0); r) +\frac{s(x_0)f(\iota(x_0); r)}{6\vM}  \\ \nonumber
&  \quad\qquad\qquad \qquad  \qquad \qquad + \frac{d(d+2)\omega(x_0)f(\iota(x_0); r)}{24 \vM}
\
\Big)r^{d+2} \\ 
 &\qquad +\Oh( f(\iota(x_0); r) r^{d+3}). 
\end{align}
\end{lemma}
\begin{proof}
This a slight variation of Lemma~B.5 of \citep{wu2018think} for a uniform density and in which the functions are allowed to depend on the parameter $r$. Because of this modification the asymptotic expansion has been slightly refined. 
\end{proof}
\begin{lemma}
\label{lem: LocCov}
For fixed $x_0 \in \MC$, $r$ sufficiently small and $X$ uniformly distributed on $\MC$ under Assumptions~\ref{assum: Rot} and ~\ref{assum: Assum2}, it holds that 
\begin{eqnarray*}
\lefteqn{\expec\left( f(X;r) (X- \iota(x_0))  (X- \iota(x_0))^\top  \1\{ \lVert X - \iota(x_0)\rVert\le r \}  \right)} \\
& \hspace{35mm}  = \frac{\lvert S^{d-1}\rvert}{d(d+2) \vol(\MC)} f(\iota(x_0);r)  r^{d+2}  
 \left(
 \begin{pmatrix} 
 I_{d\times d} & 0 \\ 0 & 0
 \end{pmatrix} 
 + \Oh(r^2)\right)
 \end{eqnarray*}
\end{lemma}
\begin{proof}
The proof follows along the same lines as Proposition~3.1 in \citet{wu2018think}.
\end{proof}

\begin{lemma}
\label{lem: FourthOrd}
For fixed $x_0 \in \MC$, $r$ sufficiently small and $X$ uniformly distributed on $\MC$ under Assumptions~\ref{assum: Rot} and ~\ref{assum: Assum2}, it holds that 
\begin{align*}
\expec \left[
f(X;r) e_k^\top(X-\iota(x_0))(X-\iota(x_0))^\top e_l e_m^\top(X-\iota(x_0))(X-\iota(x_0))^\top e_n \1\{ \lVert X - \iota(x_0)\rVert\le r \}
\right]\\
=
\frac{f(\iota(x_0);r)}{(d+4)\vM} r^{d+4} C_{k,l,m,n} + \Oh(r^{d+5}),
\end{align*}
where 
\[
C_{k,l,m,n} = \int_{S^{d-1}} \langle \iota_*\theta , e_k \rangle
\langle \iota_*\theta , e_l \rangle
\langle \iota_*\theta , e_m \rangle
\langle \iota_*\theta , e_n \rangle \diff \theta.
\]
\end{lemma}
\begin{proof}
First set 
\[
\tilde{B}_r(x_0):=\iota^{-1}(B_r^{\reals^p}(\iota(x_0))\cap \iota(\MC)).
\]
Then, the quantity of interest can be written
\begin{align*}
\mathcal{I}:=\frac{1}{\vM}\int_{\tilde{B}_r(x_0)} 
\langle \iota(y)-\iota(x_0) , e_k \rangle
&\langle \iota(y)-\iota(x_0) , e_l \rangle\\
& \times \langle \iota(y)-\iota(x_0) , e_m \rangle
\langle \iota(y)-\iota(x_0) , e_n \rangle f(y)  \diff V(y).
\end{align*}
Recalling that for $(t,\theta )\in [0, \infty)\times S^{d-1}$
\begin{align*}
 \iota \circ \exp_{x_0} (\theta t) -\iota(x_0) &= \iota_* \theta t +\Oh(t^2)\\
\tilde r &= r + \Oh(r^3)\\
\diff V(\exp_{x_0} (\theta t))&= t^{d-1} + \Oh(t^{d+1})\\
f(\exp_{x_0} (\theta t))&= f(x_0) + \Oh(t),
\end{align*}
it holds that 
\begin{align*}
\mathcal{I}&= \frac{1}{\vM} \int_{S^{d-1}}\int_0^{\tilde r} f(x_0) t^{d+3} \langle \iota_*\theta , e_k \rangle
\langle \iota_*\theta , e_l \rangle
\langle \iota_*\theta , e_m \rangle
\langle \iota_*\theta , e_n \rangle + \Oh(t^{d+4}) \diff t \diff \theta \\
& =\frac{f(x_0)}{\vM(d+4)} r^{d+4}\int_{S^{d-1}}\langle \iota_*\theta , e_k \rangle
\langle \iota_*\theta , e_l \rangle
\langle \iota_*\theta , e_m \rangle
\langle \iota_*\theta , e_n \rangle \diff \theta + \Oh(r^{d+5}),
\end{align*}
as claimed.
\end{proof}

\begin{proof}[Proof of Lemma~\ref{lem: OptPot}]
Using Lemma~\ref{lem: B5} with $f = 1$, we get that the quantile function of the local distribution of squared distances at $x_0$ is approximately 
\[
p \mapsto \left(  \frac{p}{\lvert S^{d-1} \rvert d^{-1} \vol(\MC)^{-1}}\right)^{2/d}, 
\]
so that the duality constraint in the discrete problem is approximately 
\begin{equation}
\label{eq: ApproxDual}
 \frac{N+1}{\eps} \sum_{j=0}^N \left(u^\star(X_0)+ u^\star(X_j) - \left(  \frac{U_{(j:N)}}{\lvert S^{d-1} \rvert d^{-1} \vol(\MC)^{-1}}\right)^{2/d} \right)_+=1,   
\end{equation}
where $U_{(j:N)}$ is the $j$-th sorted element of an i.i.d.\ sample of size $N$ of random variables uniformly distributed on $[0,1]$ and we take $U_{(0:N)} = 0$. Note that although we write down all $N$ order statistics $U_{(j:N)}$ and the expression for the quantile function is only a good approximation for $p \ll 1$, as long as $u(x_0) + u(X_j)$ is small, only the first few terms will be nonzero.

We get the equivalent problem
\begin{equation}
\label{eq: ApproxDisc}
\frac{N+1}{\eps (N+1)^{2/d}}\ \kappa_d\ \sum_{j=0}^N \left( \tilde{u}(x_0)+ \tilde{u}(X_j) - U_{(j:N)}^{2/d} \right)_+=1, 
\end{equation}
where we have set
 \[
\tilde{u}(\cdot) = \frac{(N+1)^{2/d}}{\kappa_d} u^\star(\cdot).
\]
 Choose $k^{2/d} < 2 \tilde{u} \leq (k+1)^{2/d} $ and plug $\tilde{u}$ (choosing a constant approximation to the potential) as a choice for the potential in \eqref{eq: ApproxDisc}. 
 It yields, 
\begin{equation}
\label{eq: ConstDisc}
 \frac{N+1}{\eps (N+1)^{2/d}}\ \kappa_d\ \left[ 2(k+1)\tilde{u} - \sum_{j=1}^k U_{(j:N)}^{2/d} \right]=1.
\end{equation}

Let us turn to the size of the sum in  \eqref{eq: ConstDisc}.
First, basic calculations show that 
\[
\expec \left[ \left (U_{(j:N)}\right)^{2/d} \right] = \frac{\Gamma(\tfrac2d + j)\Gamma(N+1)}{\Gamma(\tfrac2d + N+1 )\Gamma( j)},
\]
So that understanding the problem \eqref{eq: ApproxDual}, even in expectation and for constant potentials is not so easy for $d>2$. 

For $d=1$, we get 
\[
\sum_{j=1}^k \expec \left (U_{(j:N)}\right )^{2/d} = \frac{1}{(N+2)(N+1)} \left [ \frac13  k(k+1)(k+1) \right],
\]
while for $d=2$, it holds that 
\[
\sum_{j=1}^k \expec \left (U_{(j:N)}\right )^{2/d} = \frac{1}{(N+1)} \left [ \frac12  k(k+1) \right].
\]
In general,  
\[
\sum_{j=1}^k \expec \left (U_{(j:N)}\right )^{2/d} \approx  (N+1)^{-2/d} \sum_{j=1}^k j^{2/d}\left(1+ \frac{2-d}{d^2 j} +\Oh\left(\frac{1}{d^2}\right)\right),
\]
so that the leading order is 
\[
(N+1)^{-2/d} \frac{d}{d+2} k^{\frac{d+2}{d}}.
\]
Equation~\eqref{eq: ConstDisc} then becomes, 
\[
\frac{\eps N^{2/d -1}}{\kappa_d} \approx \left(1 - \frac{d}{d+2}\right) k^{\frac{d+2}{d}} = \left( \frac{2}{d+2}\right) k^{\frac{d+2}{d}}
\]
so that
\[
\tilde u \approx \frac12   \left( \frac{2+d}{2\kappa_d}\right)^{\frac{
2}{d+2}} \eps^{\frac{2}{d+2}} N^{\frac{(2-d)2}{d(d+2)}} 
\]
and then 
\[
  u \approx  \kappa_d^{\frac{d}{d+2}} \left( \frac{d+2}{2}\right)^{\frac{
2  }{d+2} } \eps^{\frac{2}{d+2}} N^{\frac{(2-d)2}{d(d+2)}-2/d} =    \left(  \frac{\vol(\MC) d}{\lvert S^{d-1}\rvert} \ \frac{(d+2)}{2}\right)^{\frac{
2}{d+2}} \eps^{\frac{2}{d+2}} N^{\frac{-4}{(d+2)}}
\]
Finally, using this first order to approximately solve the equation yields the claim.
\end{proof}
\begin{proof}[Proof of Theorem~\ref{thm: MainConv}]
As $g$ is twice differentiable, we can write 
\[
g(X_j) = g(x_0) + L_0 (x_0-X_j) + \frac12 (X_j -x_0)^\top Q_0  (X_j -x_0) + \oh\left ( \lVert x_0 -X_j \rVert^2 \right),
\]
where $L_0$ is the gradient of $g$ at $x_0$ and $Q_0$ is the Hessian of $g$ evaluated at $x_0$.
Plugging this result in the definition of $( \Delta^{OT} g)(x_0)$, we derive 
\begin{align*}
&( \Delta^{OT} g)(x_0)  \\
& \qquad 
= \sum_{j=1}^N \overline{W}_{0,j}^{(f)} \left( - L_0 (X_0-X_j) - \frac12 (X_j -X_0)^\top Q_0  (X_j -X_0) + \oh\left ( \lVert X_0 -X_j \rVert^2 \right)
\right).
\end{align*}
We will split this sum into three terms and control each one separately. We will first consider the expectation and then the variance. We rewrite $W_{0,j}^\eps :=\overline{W}_{0,j}^{(f)}$ to streamline the notation in the proof. Recall that $X_0=\iota(x_0)$ is a given point.\\
\underline{\textit{Step~1: Expectation.}}

Let us start with the second term 
\[
- \frac12\sum_{j=0}^N W_{0,j}^\eps  (X_j -X_0)^\top Q_0  (X_j -X_0) =: B
\]
The quantity in the above display is a scalar so that it is equal to its trace. Further, using the linearity and the cyclical property of the trace, it holds that
\begin{align*}
B &= - \frac12 \tr \left [ Q_0 \sum_{j=0}^N W_{0,j}^\eps    (X_j -X_0) (X_j -X_0)^\top \right ] \\
&= - \frac12 \tr \left [ Q_0 \sum_{j=0}^N \frac{( K_{\eps, N} -c_{0,j})_+}{\eps} (X_j -X_0) (X_j -X_0)^\top  \right   ]. 
\end{align*} 
Using Lemma~\ref{lem: LocCov}, it holds that 
\begin{align*}
&\expec \sum_{j=1}^N \frac{( K_{\eps, N} -c_{0,j})_+}{\eps} (X_j -X_0) (X_j -X_0)^\top 
\\
& \qquad\qquad\qquad = 
\frac{N}{\eps} \frac{\lvert S^{d-1}\rvert}{d(d+2) \vol(\MC)} K_{\eps, N}^{\frac{d+2}{2}} K_{\eps, N}   
 \left(
 \begin{pmatrix} 
 I_{d\times d} & 0 \\ 0 & 0
 \end{pmatrix} 
 + \Oh(K_{\eps,N})\right).
\end{align*}
 
Let us now address the first term, i.e., 
\[
-L_0 \sum_{j=0}^N W_{i,j}^\eps   (X_0-X_j).
\]
The second part of Lemma~B.5 in \citet{wu2018think} reads, in our case,
\begin{align*}
&\expec \left[ (X-\iota(x_0)) f(X;r) \1\{ \lVert X - \iota(x_0)\rVert\le r \}  \right]  \\
& \qquad \qquad =  \frac{\lvert S^{d-1} \rvert}{(d+2)\vol(\MC)} \left\llbracket \frac{ J_{p,d}^\top\iota_*\nabla f(X_0;r)}{d},
\frac{f(X_0;r) \tilde{J}_{p,p-d}^\top \mathfrak{N}(x_0)}{2}
\right\rrbracket r^{d+2} + \Oh(r^{d+4}).
\end{align*}
It follows that 
\begin{align*}
&\expec \sum_{j=1}^N W_{0,j}^\eps   (X_0-X_j) \\
&\qquad\qquad = \frac{N}{\eps} \frac{\lvert S^{d-1} \rvert}{(d+2)\vol(\MC)} \left\llbracket  0 ,
\frac{K_{\eps, N} \tilde{J}_{p,p-d}^\top \mathfrak{N}(x_0)}{2}
\right\rrbracket K_{\eps,N}^{\frac{d+2}{2}}   \\
&\qquad\qquad\quad + \Oh\left( \frac{N}{\eps}K_{\eps,N}^{(d+4)/2}\right).
\end{align*}

In view of the developments above, the expectation of the Taylor residual is negligible. \\
\underline{\textit{Step~2: Variance.}}

Let us deal with \[
-L_0 \sum_{j=1}^N W_{0,j}^\eps   (X_0-X_j).
\]
We have that 
\[\var\sum_{j=1}^N W_{0,j}^\eps   (X_0-X_j) = \frac{N}{\eps^2} \var ( ( K_{\eps, N} -c_{0,j})_+ (X -X_0))
\]
Further, 
\begin{align*}
&\var \left[ \big( K_{\eps, N} -c(X,X_0)\big)_+ (X -X_0)\right] \\
& \qquad\qquad\qquad = \expec \left[\big( K_{\eps, N} -c(X,x_0)\big)_+^2 (X -X_0)(X -X_0)^\top\right] \\
&\qquad\qquad\qquad- \expec \left[\big( K_{\eps, N} -c(X,x_0)\big)_+ (X -X_0)\right] \expec \left[\big( K_{\eps, N} -c(X,X_0)\big)_+ (X -X_0)^\top\right]\\
&\qquad\qquad\qquad= \frac{\lvert S^{d-1}\rvert}{d(d+2) \vol(\MC)} \kappa_d^{\frac{d+2}{2}} K_{\eps, N}^2  \eps
 N^{-2}  
 \left(
 \begin{pmatrix} 
 I_{d\times d} & 0 \\ 0 & 0
 \end{pmatrix} 
 + \Oh(K_{\eps,N})\right) \\
 &\qquad\qquad\qquad-  \frac{\lvert S^{d-1} \rvert^2}{(d+2)^2\vol^2(\MC)} 
 K_{\eps, N}^2
 vv^\top \kappa_d^{{d+2}}  \eps^2
 N^{-4}   + \Oh\left(K_{\eps,N}^{(d+4)/2}\eps N^{-2} K_{\eps,N}\right), 
\end{align*}
where $v$ are vectors that depend on the curvature as above.
We finally get 
\[
\var \left[
\sum_{j=1}^N W_{0,j}^\eps   (x_0-X_j)
\right] = \Oh\left(  \frac{N}{\eps^2}\ K_{\eps,N}^{2}N^{-2} \eps  \right) = \Oh\left(   \frac{K_{\eps,N}^{2}}{\eps N}  \right).
\]
Let us now turn to the covariance matrix of 
\[
V:= \vect \left( \sum_j W_{i,j}^\eps (x_i-X_j)(x_i-X_j)^\top\right)
\]
which, using Equations~(1.3.14), (1.3.16) and (1.3.31) in \citet{kollo2005advanced},  is equal to 
\[
N \expec [(W_j^\eps)^2 (x_i -X_j)\otimes(x_i -X_j)^\top\otimes (x_i -X_j)\otimes(x_i -X_j)^\top ]- N \expec V\expec^\top V.
\]

Relying on Lemma~\ref{lem: FourthOrd}, we get that the leading order of the variance of $ V$ is 
\[
\Oh\left( \ \frac{N}{\eps^2}\ K_{\eps,N}^{2} K_{\eps,N}^{\frac{d+4}{2}} \right)= \Oh\left( \frac{K_{\eps, N}^3}{N\eps}  \right).
\]
The claim follows.
\end{proof}

\section{Numerical experiments}

\subsection{Simulations for optimal dual potentials}

We exhibit the behaviour of the optimal potentials for $N$ random points on the $d$-sphere and various values of the parameter $\eps$. In $d = 1, 2, 3$, $N = 10^3$ points were sampled uniformly from $S^{d}$ by sampling first a standard Gaussian in $\mathbb{R}^{d+1}$ and scaling to unit norm. We numerically solved the corresponding discrete optimal transport problem with $\varepsilon$ in the range $[10^{-2}, 10^6]$ and plotted $\log(\overline{u})$ against $\log(\varepsilon)$, where $\overline{u}$ denotes the averaged value of the resulting dual potential over the $N$ points. For $\varepsilon$ sufficiently large, we estimated the exponent $\alpha$ for the relationship $u \sim \varepsilon^{\alpha}$. Our empirical findings agree with the exponent $\frac{2}{2 + d}$.
\begin{figure}[h!]
	\begin{center}
		\includegraphics[width=0.32\textwidth, ]{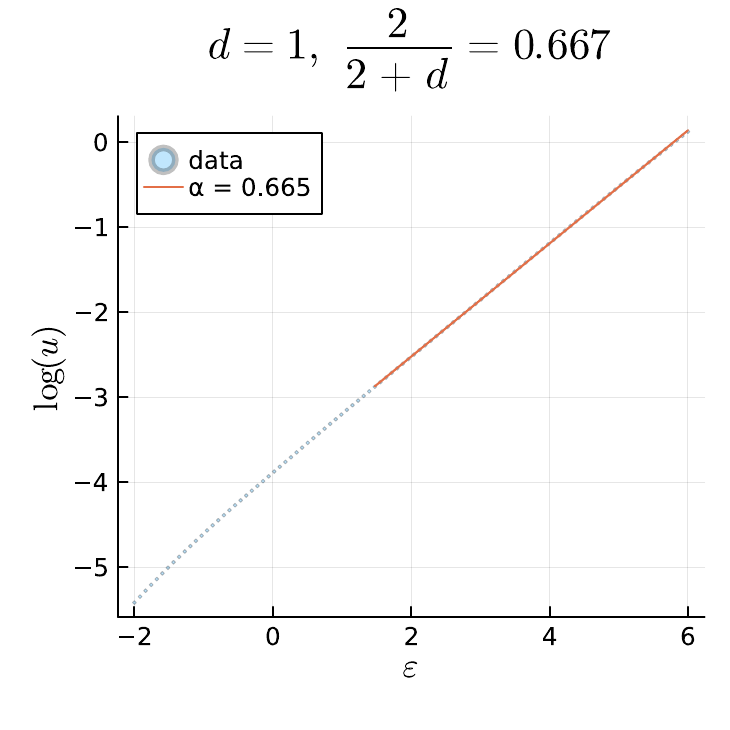}
		\includegraphics[width=0.32\textwidth, ]{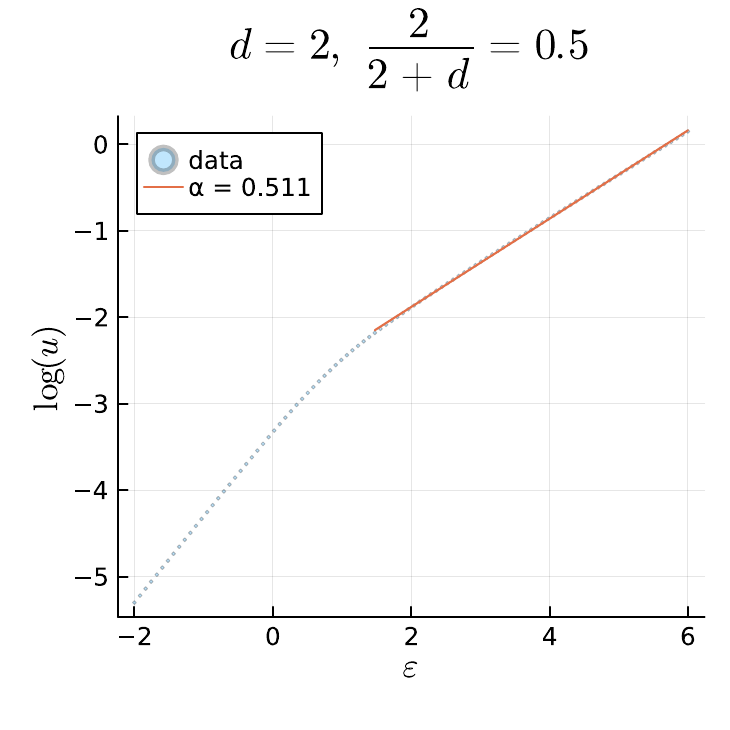}
		\includegraphics[width=0.32\textwidth, ]{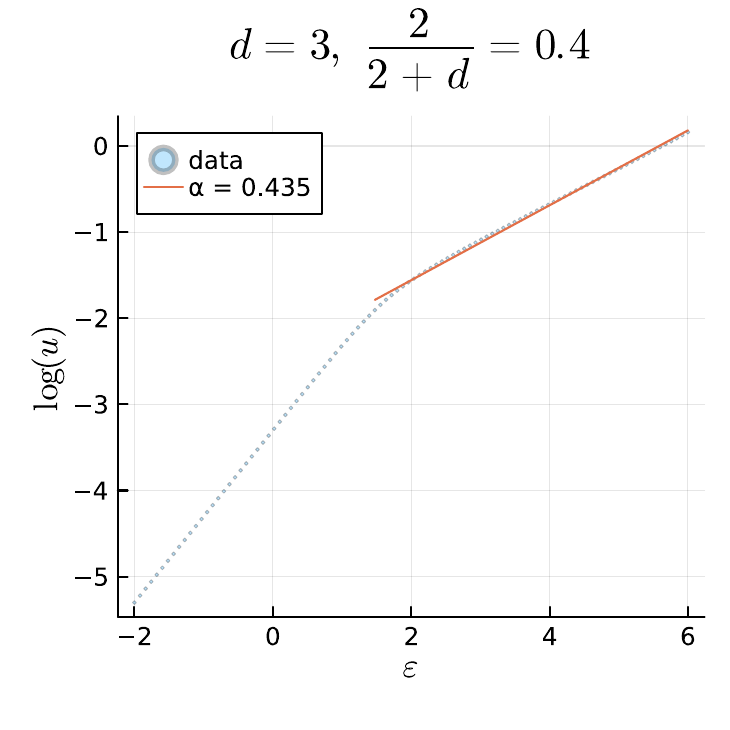}
	\end{center}
	\caption[]{Scaling of dual potential $u$ on the $d$-sphere, for $d = 1, 2, 3$. }
	\label{fig:Sphere}
\end{figure}

\subsection{Behaviour of the operator}

Next, we investigate the behaviour of the operator $\Delta^{OT}$ in the discrete setting where points are sampled from the uniform distribution on the 2-dimensional torus with major and minor radii $R = 1, r = 1/2$. We fix a point $(0, 1, 1/2)$ at which the tangent space $T_{x_0} \MC$ is spanned by $(e_1, e_2)$. We then consider a function $f(x, y, z) = \frac{1}{2}(3x^2 + 5y^2 + 7z^2)$. For $N = 100, 250, 500, 1000, 2500, 5000, 10,000$ points sampled uniformly from the torus, we calculated the $(N+1) \times (N+1)$ coupling $\pi$ by solving \eqref{eq: optimalPlan} and then computed the quantity $K_{\varepsilon, N}^{-1} (\Delta^{OT} f)(x_0)$. 

Motivated by the asymptotic scalings we derived, we tried setting $\varepsilon \propto N^\alpha$ for varying exponents: $\alpha = 2$ which should correspond to a fixed regularisation level in the continuous case (and we do not expect convergence to the Laplacian in this case), and $\alpha = 1.75, 1.5, 1.25$ which all fall within the regime where the results of Theorem \ref{thm: MainConv} apply. We show in Figure \ref{fig:Torus} the values of $K_{\varepsilon, N}^{-1} (\Delta^{OT} f)(x_0)$ over 25 independent samples at each value of $N$. 

We see that when $\alpha = 2$, the quantity $K_{\varepsilon, N}^{-1} (\Delta^{OT} f)(x_0)$ stabilises around a fixed value close to zero as $N$ increases. This agrees with our understanding that $\varepsilon \propto N^2$ in the discrete setting corresponds to the continuous case of empirical distributions with a fixed value of $\varepsilon$. On the other hand, when $1 < \alpha < 2$ we observe a pattern of values appears to converge around a different, positive, value. Importantly, for various $1 < \alpha < 2$, these values are similar -- this supports the scaling relation of Theorem 1 and suggests that the quantity is converging to the value (up to a constant independent of $\varepsilon, N$) of the Laplace--Beltrami operator at $x_0$. As a local averaging process takes place, setting $\eps$ too small means implies smaller neighbourhoods and thus a larger variances as is exhibited on the bottom right panel of Figure~\ref{fig:Torus}.

\begin{figure}[h!]
	\begin{center}
    	\includegraphics[width = 0.325\linewidth]{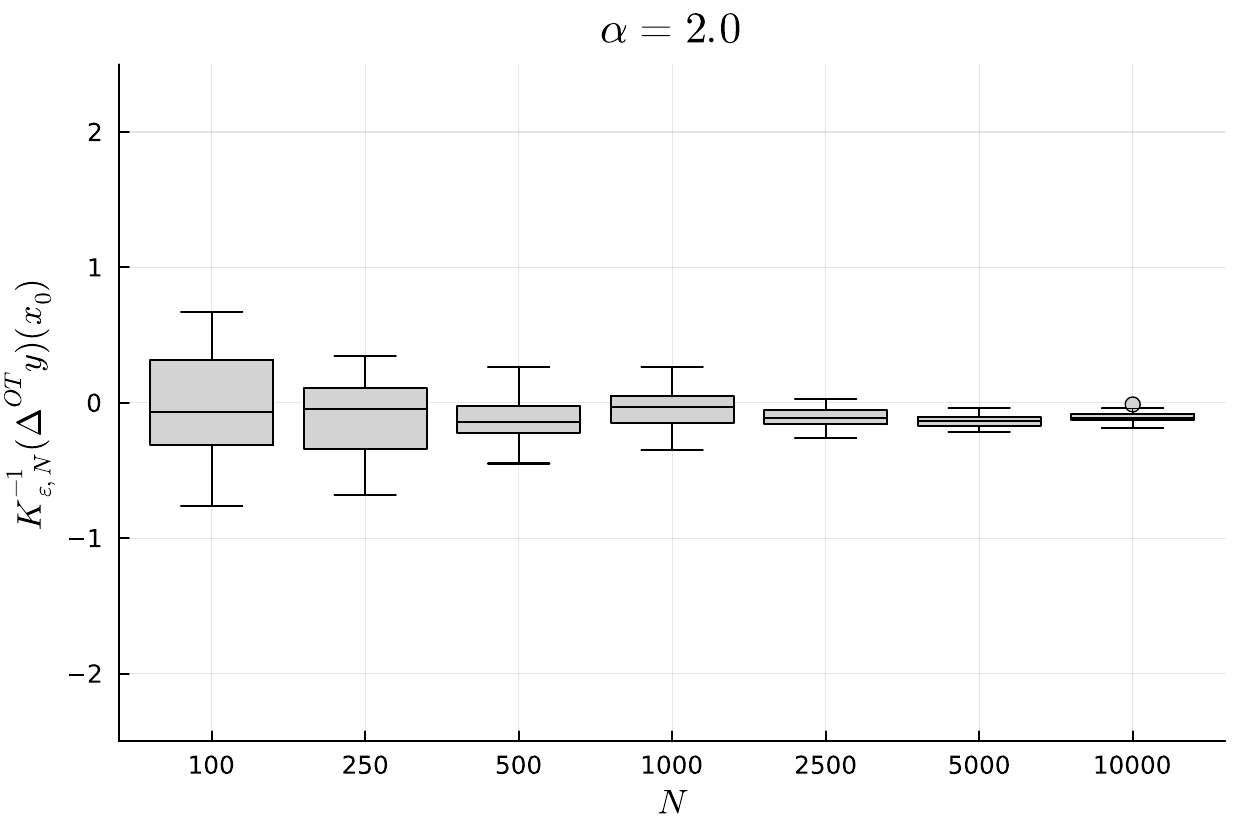}
    	\includegraphics[width = 0.325\linewidth]{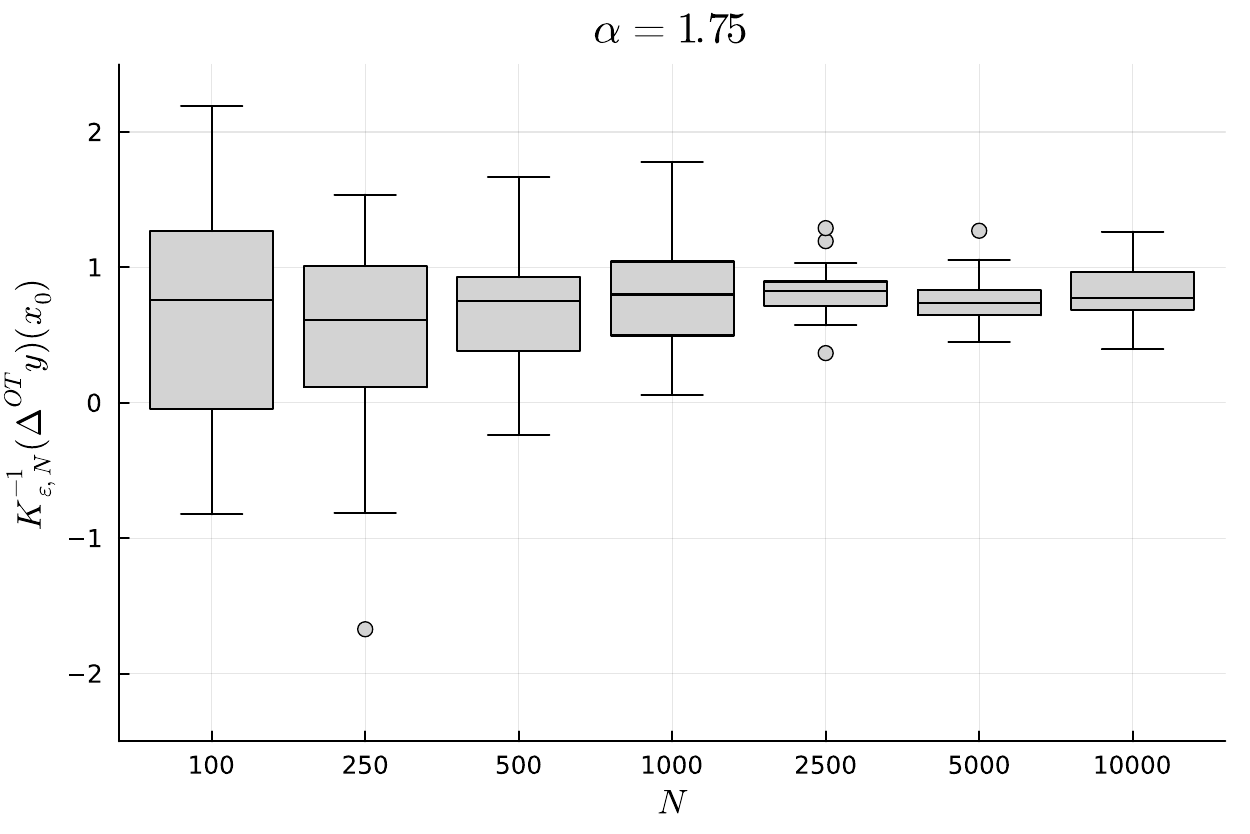} \\ 
    	\includegraphics[width = 0.325\linewidth]{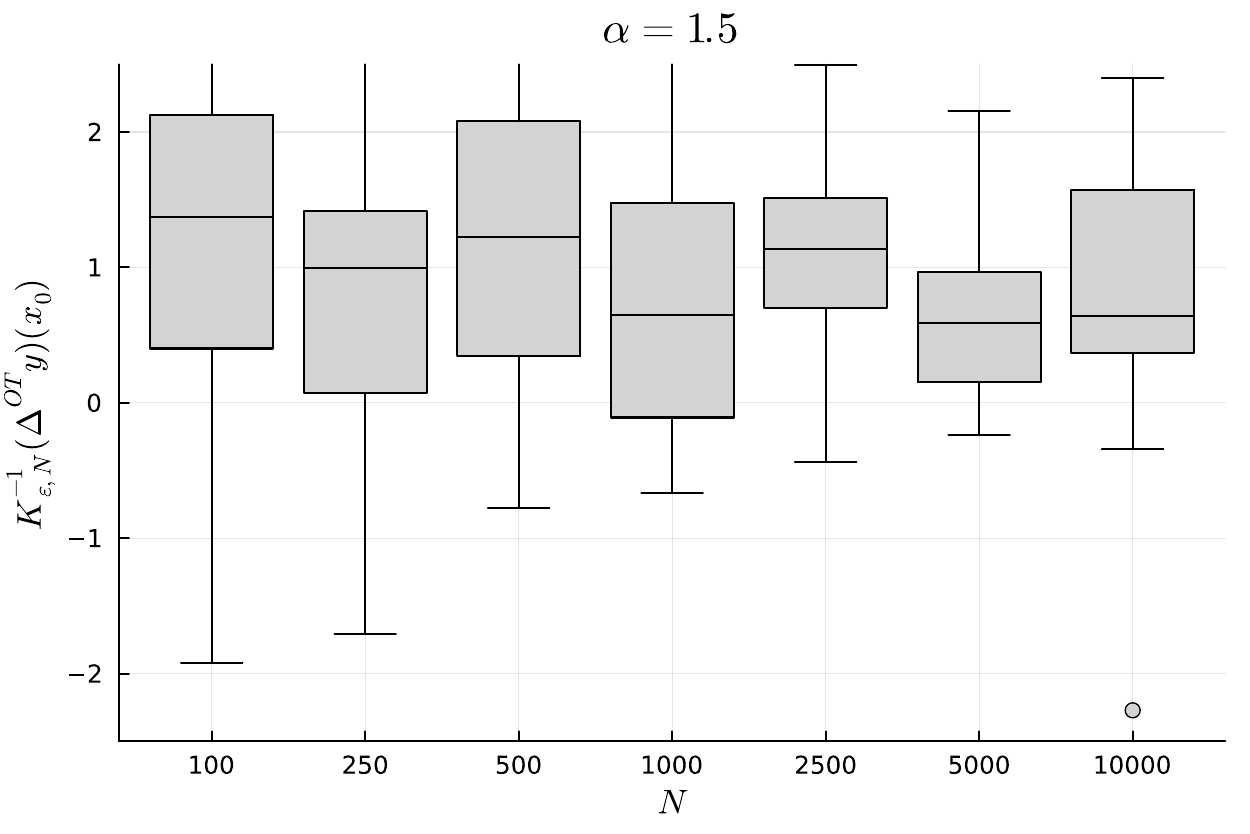}
    	\includegraphics[width = 0.325\linewidth]{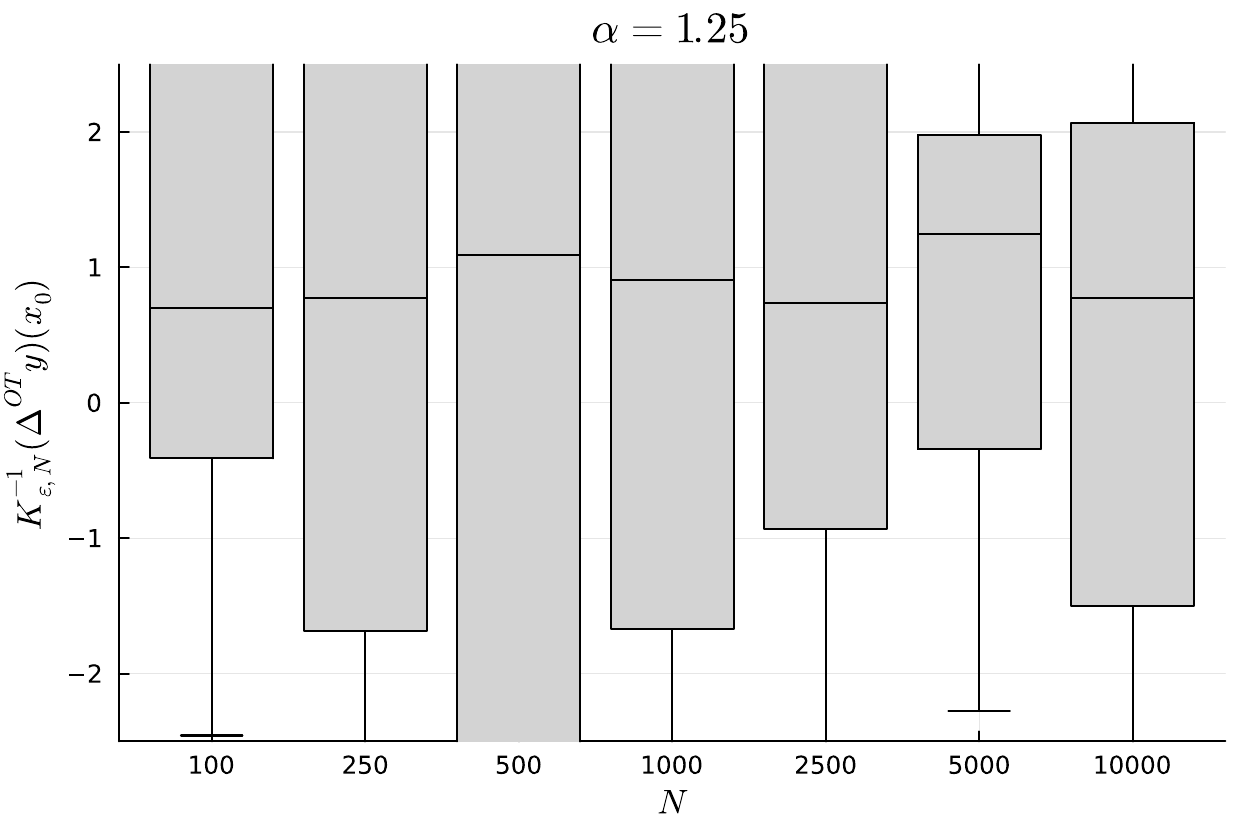} 
	\end{center}
	\caption[]{Estimate of Laplacian $K_{\varepsilon, N}^{-1} (\Delta^{OT} f)(x_0)$ for varying sample sizes $N$, and $\varepsilon \propto N^\alpha$, for various choices $  \alpha = 2, 1.75, 1.5, 1.25$.}
	\label{fig:Torus}
\end{figure}

\section{Computational scheme}

\paragraph{Symmetric semi-smooth Newton}  To solve the dual problem \eqref{eq: DualLor}, the algorithm proposed by \citet[Algorithm 2]{lorenz2021quadratically} can be adapted to the symmetric case in which case a single dual potential is solved for. 
This algorithm is a semi-smooth Newton algorithm, which explains the fact that one solves the linear system \eqref{eq:lorenz_linear_system} to obtain the update $\Delta u$.
We detail this below in Algorithm~\ref{alg:lorenz_symmetric}, and notably the symmetry halves the size of the linear system \eqref{eq:lorenz_linear_system}. Since we seek a projection onto the set of \emph{hollow} bistochastic matrices and not the set of bistochastic ones, the dual problem involves a term $\1 \1^\top - I$. Remark that, as in the dual formulation \eqref{eq: DualHollow}, the hollowness constraint can be achieved by putting $\infty$ on the diagonal of the cost matrix.

A simpler algorithm based on alternating projections onto the subspace of symmetric matrices with prescribed row-sums and the cone of non-negative matrices was proposed in \citep[Algorithm 1]{zass2006doubly}. We note that generally such schemes are only guaranteed to converge to a \emph{feasible} point in the intersection of two convex sets \cite{chu2023structured}, not necessarily the optimal one.  
  In Figure \ref{fig:convergence_comparison} we compare its convergence behaviour to the semi-smooth Newton method (Algorithm \ref{alg:lorenz_symmetric}). For $N = 250, 1000$, starting with a $N \times N$ matrix of standard Gaussians, we symmetrise it and compute its bistochastic projection in Frobenius norm using either algorithm. As a measure of convergence, we record at each iteration the violation of the bistochastic constraint: $\| \pi \ones - \ones \|_2$. In both the cases we consider, the semi-smooth Newton method converges in fewer than 10 iterations. On the other hand, the alternating scheme takes several orders of magnitude more iterations, and its convergence rate appears to become much slower for larger problem sizes. The much faster convergence of the semismooth Newton method is characteristic of second-order methods which typically exhibit quadratic convergence, while alternating projections typically exhibit linear convergence. This behaviour is consistent with the observations of \cite{chu2023structured}, who propose a structured BFGS method for solving the dual QOT problem.

\begin{algorithm}
    \caption{Algorithm 2 of \cite{lorenz2021quadratically}, specialised to symmetric inputs}
    \label{alg:lorenz_symmetric}
    \begin{algorithmic}
        \State Input: cost matrix $C_{ij} \ge 0$, regularisation parameter $\varepsilon > 0$, Armijo parameters $\theta, \kappa \in (0, 1)$, conjugate gradient regulariser $\delta = 10^{-5}$. 
        \State Initialise: $u \gets \1$
        \State \begin{equation}
            \Phi(u) = -\langle u, \1 \rangle + \frac{1}{4\varepsilon} \| [u \1^\top + \1 u^\top - C]_+ \|_2^2.
        \end{equation}
        \While{not converged}
            \State $P_{ij} \gets u_i + u_j - C_{ij}$
            \State $\sigma_{ij} = \1\{ P_{ij} \ge 0 \}$
            \State $\pi_{ij} = \max\{ P_{ij}, 0 \}/\varepsilon$
            \State Solve for $\Delta u$:
            \begin{equation}
                \left( \sigma + \diag{\sigma \1} + \delta I \right) \Delta u = -\varepsilon (\pi - I)\1.\label{eq:lorenz_linear_system}
            \end{equation}
            \State Set $t = 1$ and compute 
            \[
            d= \eps \sum_{ij}\pi_{ij}\left[ (\Delta u)_i +(\Delta u)_j\right] \ - 2 \eps \langle \Delta u, \mu\rangle
            \]
            \While{$\Phi(u + t \Delta u) \ge \Phi(u) + t \theta d$} 
                \State $t \gets \kappa t$
            \EndWhile
            \State $u \gets u + t \Delta u$
        \EndWhile 
        \State $\pi_{ij} = \max\{ 0, u_i + u_j - C_{ij}\}/\varepsilon$
    \end{algorithmic}
  \end{algorithm}

\begin{figure}
  \centering 
  \begin{subfigure}{0.4\textwidth}
    \includegraphics[width = \linewidth]{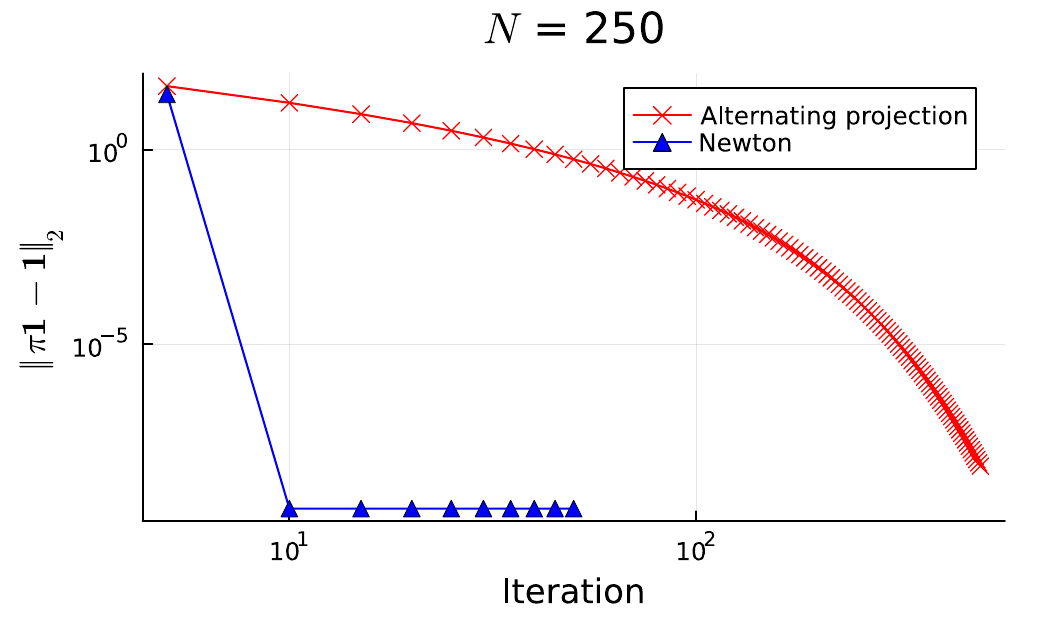}
  \end{subfigure}
  \begin{subfigure}{0.4\textwidth}
    \includegraphics[width = \linewidth]{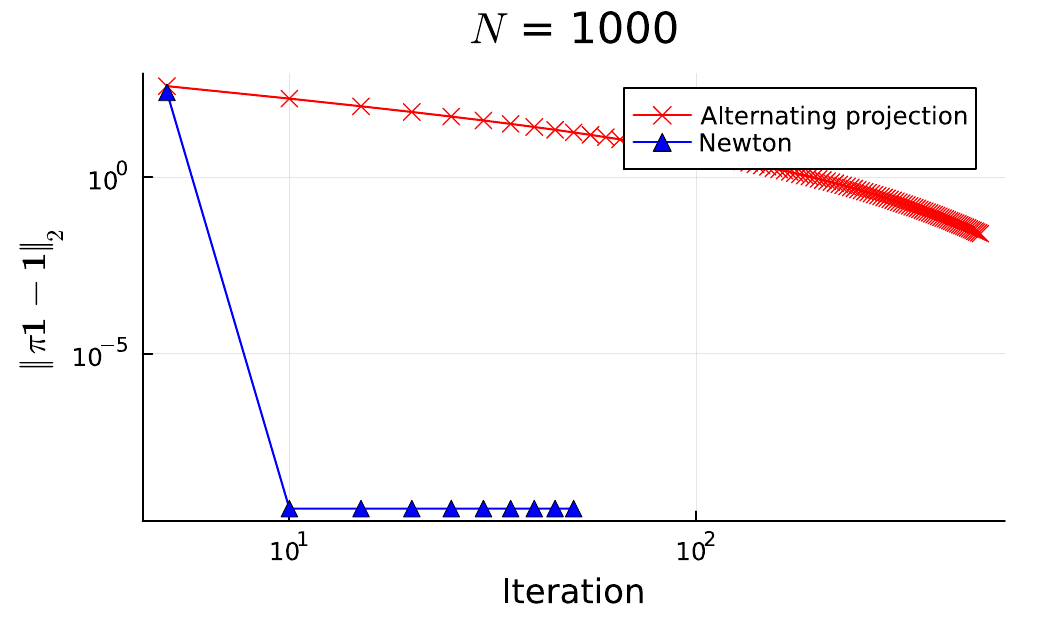}
  \end{subfigure}
    \caption{Comparison of alternating projections to symmetric semi-smooth Newton. Bistochastic constraint violation (measured in $L^2$, i.e., $\| \pi \ones - \ones \|_2$) vs. number of iterations. }
    \label{fig:convergence_comparison}
\end{figure}

\paragraph{Active-set method} For large datasets, we show that this method is amenable to an active set method like the one proposed in \citep{lim2020doubly}. This allows us to take advantage of the sparse nature of the bistochastic projection in Frobenius norm: most computations are done explicitly in terms of sparse matrices. We introduce a sparse support matrix $S \in \{ 0, 1\}^{N \times N}$ and write $S^c = 1 - S$. Consider then the primal problem with support restricted to $S$:
\begin{align}
    \min_{A \in \KCo, \: A \odot S^c = 0}  \| A + \varepsilon^{-1} C \|_F^2. \label{eq:primal-as}
\end{align}
If the feasible set for the restricted-support primal problem \eqref{eq:primal-as} has nonempty relative interior, the corresponding dual problem exists and strong duality holds, e.g., by Slater's condition. The corresponding dual problem is written:
\begin{align}
  \max_{u} \: \langle u, \1 \rangle - \frac{1}{4 \varepsilon} \left\| [u \1^\top + \1 u^\top -  C]_+ \odot S \right\|_F^2. \label{eq:dual-as}
\end{align}
In the above, we have absorbed the hollow constraint on the diagonal into the support matrix $S$ by assuming that $S_{ii} = 0$ for all $i$.
The problem \eqref{eq:dual-as} can be tackled using the same semi-smooth Newton algorithm as proposed in Algorithm \ref{alg:lorenz_symmetric}. The restriction on the support of $A$ can be formally understood as setting $C_{ij} = +\infty$ whenever $S_{ij} = 0$. The benefit of this approach is that solution of each restricted-support problem by the semismooth Newton algorithm \ref{alg:lorenz_symmetric} involves only $\Oh(|S|)$ nonzero entries, rather than $\Oh(N^2)$ in the dense case. For $N$ larger than a few thousand, we find that this leads to significant speedups. As proved by the analysis in \citep[Lemma 1; Proposition 2]{lim2020doubly}, it is sufficient to use bistochasticity of the primal solution as a stopping condition for Algorithm \ref{alg:activeset} and convergence within a finite number of steps is guaranteed. The critical factor for good behaviour in Algorithm \ref{alg:activeset} is the choice of initial support set $S$, in particular for duality to hold the support set $\{ A : A \in \KCo, A \odot S^c \}$ must remain feasible. This is discussed at length in \cite[Section A.3]{lim2020doubly}; in practice we find good results by initialising $S$ with $k$-nearest neighbours as an initial rough guess, and adding random permutation matrices to ensure feasibility.

We demonstrate in Figure \ref{fig:activeset_benchmark} the practical computational speedup achieved by implementation of Algorithm \ref{alg:activeset}.
We sample $N = 10^3, 5\times 10^3, 10\times 10^3, 25\times 10^3$ points from a standard Gaussian distribution in dimension $d = 250$ and form the matrix $C$ of pairwise squared Euclidean distances. Following our recommended heuristic, we normalise $C$ to have mean 1, and use either Algorithm \ref{alg:lorenz_symmetric} or \ref{alg:activeset} to compute the bistochastic Frobenius projection with parameter $\varepsilon = 1$. For Algorithm \ref{alg:activeset}, we initialise the support using $k$-nearest neighbours with $k = 50$.  For each problem instance and algorithm, we record the average total runtime and memory allocation across 10 calls. Figure \ref{fig:activeset_benchmark} shows the distribution over 10 independent problem instances. For the active-set Algorithm \ref{alg:activeset} we stress that the time and memory taken to initialise the support using $k$-NN is accounted for in our results. We find that the active-set solver achieves a very significant speed and memory advantage for problems sizes of about $N = 10^4$ and larger.

For reference, an implementation of the algorithms described in this article along with Julia code for reproducing numerical experiments is available at \url{https://github.com/zsteve/QROT}. 

\begin{algorithm}
  \caption{Active-set method based on \cite[Algorithm 2]{lim2020doubly}}
  \label{alg:activeset}
  \begin{algorithmic}
    \State Input: initial support set $S_{ij} \in \{ 0, 1\}$, cost matrix $C_{ij} \ge 0$, regularisation parameter $\varepsilon > 0$, Armijo parameters $\theta, \kappa \in (0, 1)$, conjugate gradient regulariser $\delta = 10^{-5}$.
    \State $u^0 \gets \1$
    \State $S^0 \gets S$
    \State $k \gets 1$
    \While{$\pi^k$ is not bistochastic}
      \State Solve for dual potential $u^k$ \eqref{eq:dual-as} using Algorithm \ref{alg:lorenz_symmetric}, started from $u^{k-1}$ and support set $S^{k-1}$.
      \State $\pi^k_{ij} \gets \max\{ 0, u^k_i + u^k_j - C_{ij}\}/\varepsilon$
      \State $S^k \gets S^{k-1} \cup \mathrm{supp}(\pi^k)$ 
      \State $k \gets k + 1$ 
    \EndWhile
  \end{algorithmic}
\end{algorithm}

\begin{figure}
    \centering
    \begin{subfigure}{0.7\textwidth}
        \includegraphics[width = \linewidth]{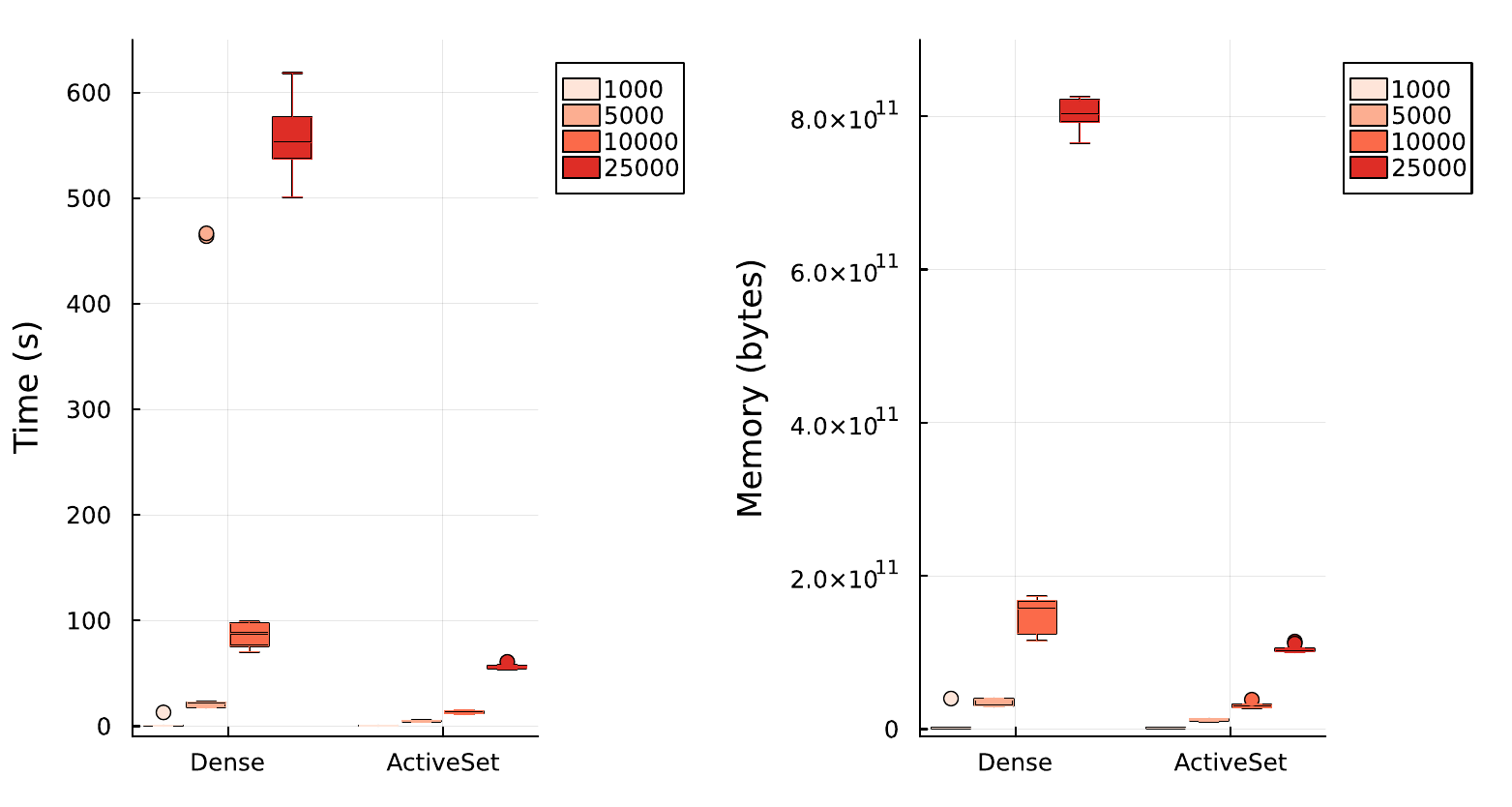}
    \end{subfigure}
    \caption{Benchmark of dense and active set algorithms. Left: algorithm runtime in seconds. Right: total memory allocation in bytes. All results are averaged over 10 repeats shown across 10 randomly generated problems for varying $N = 1000, 5000, 10000, 25000$, $d = 250$ and fixed $\varepsilon$. }
    \label{fig:activeset_benchmark}
\end{figure}

\begin{figure}
    \centering
    \begin{subfigure}{0.49\textwidth}
        \includegraphics[width = \linewidth]{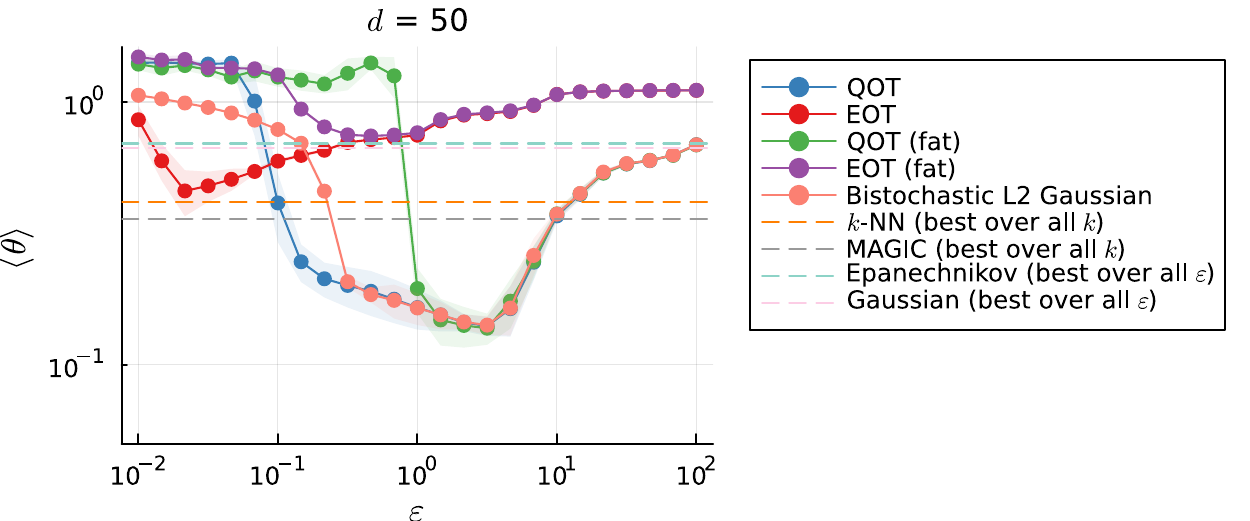}
        \caption{}
    \end{subfigure}
    \begin{subfigure}{0.49\textwidth}
        \includegraphics[width = \linewidth]{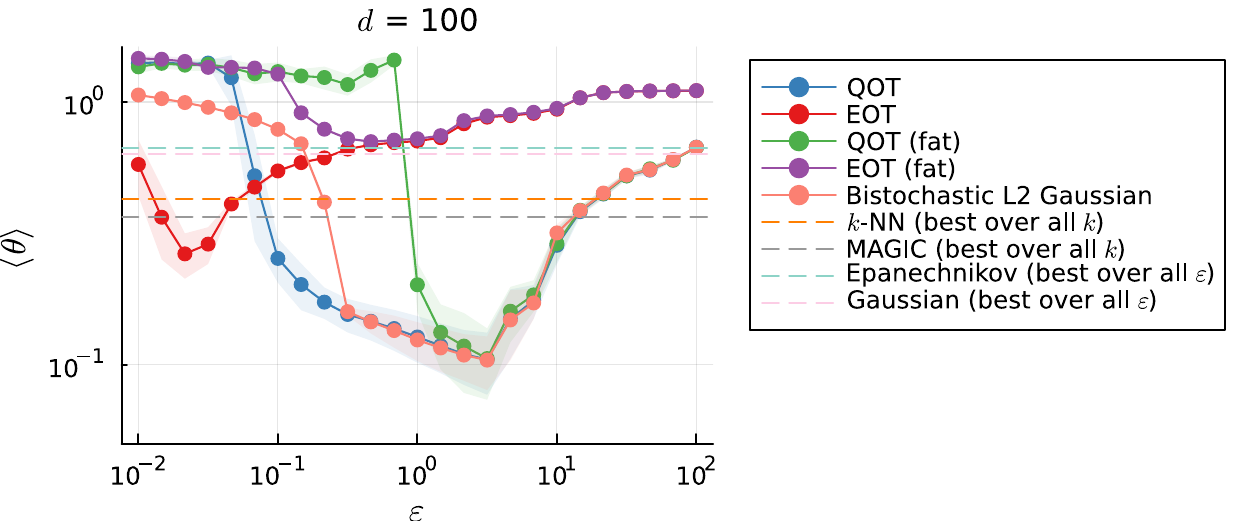}
        \caption{}
    \end{subfigure}
    \begin{subfigure}{0.49\textwidth}
        \includegraphics[width = \linewidth]{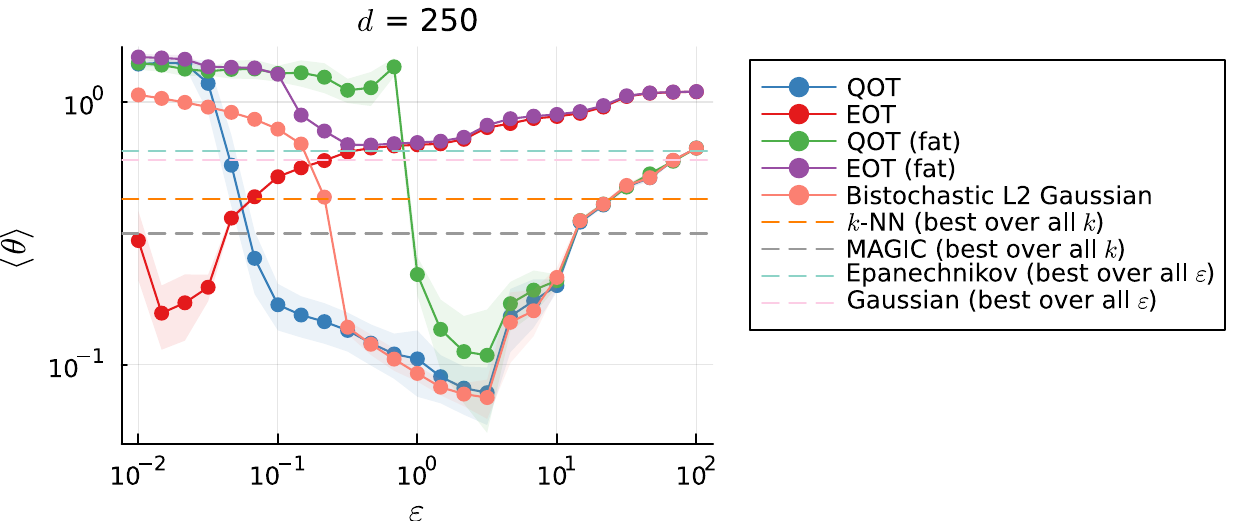}
        \caption{}
    \end{subfigure}
    \begin{subfigure}{0.49\textwidth}
        \includegraphics[width = \linewidth]{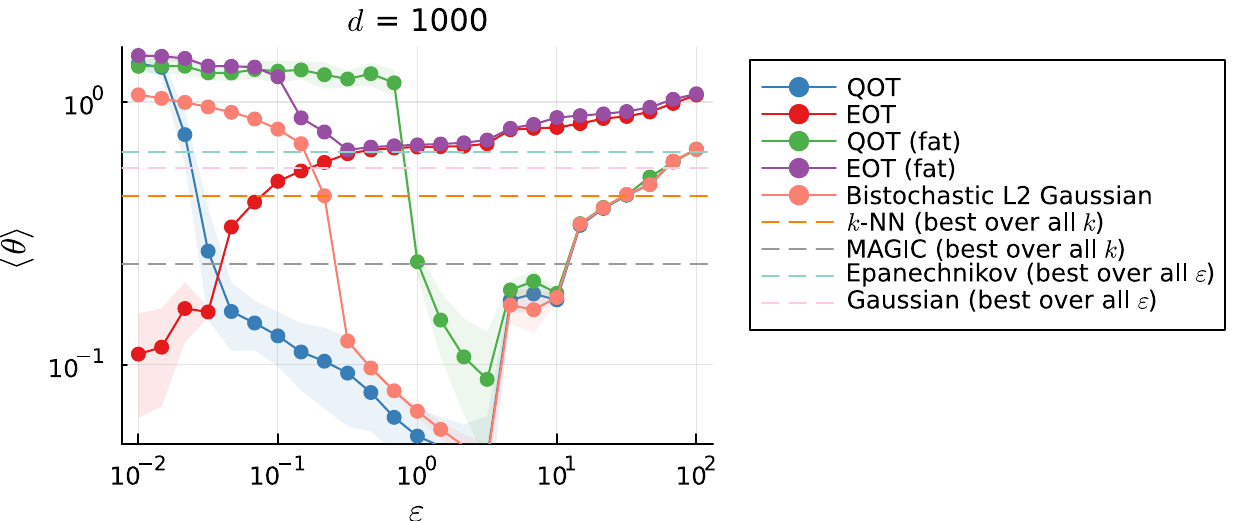}
        \caption{}
    \end{subfigure}
    \caption{Performance of different kernel constructions for spiral example as $d$ varies, measured in terms of the average principal angle between the leading Laplacian eigenspaces. Note that performance of bistochastic, hollow methods (both L2 and entropic) improves as $d$ increases. }
    \label{fig:spiral_supp_fig}
  \end{figure}

\begin{figure}
    \centering
    \begin{subfigure}{0.3\textwidth}
        \includegraphics[width = \linewidth]{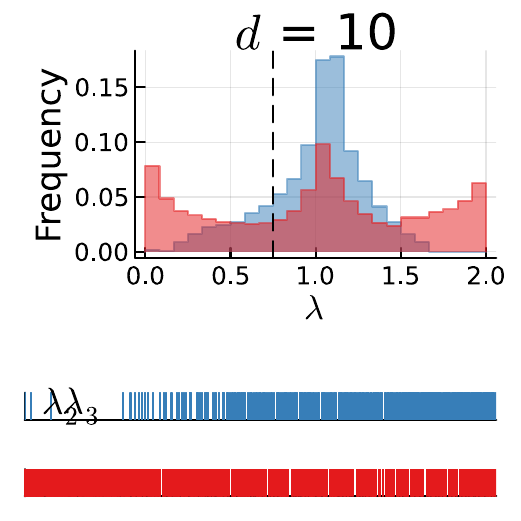}
        \caption{}
    \end{subfigure}
    \begin{subfigure}{0.3\textwidth}
        \includegraphics[width = \linewidth]{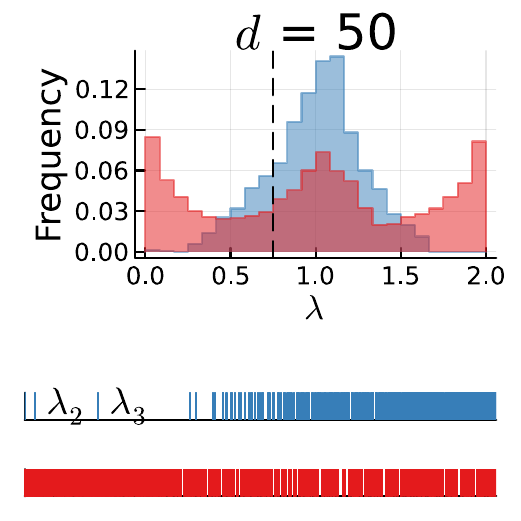}
        \caption{}
    \end{subfigure}
    \begin{subfigure}{0.3\textwidth}
        \includegraphics[width = \linewidth]{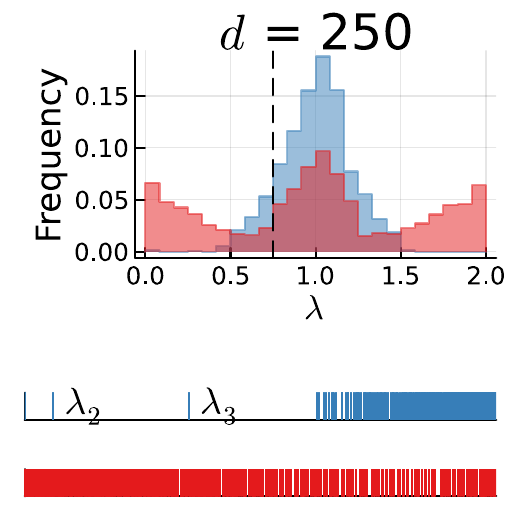}
        \caption{}
    \end{subfigure}
    \caption{Laplacian spectrum of QOT (blue) and EOT (red) affinity matrices for mixture of $k = 3$ Gaussians with varying $d$, as in Figure \ref{fig:gmm}(a-d). In all cases, $\varepsilon = 1.0\times \overline{C}$ for QOT and EOT is tuned to match the mean perplexity. }
    \label{fig:gmm_supp_fig}
\end{figure}

\begin{table}[h!]
  {
    \centering
    \begin{tabular}{|l|p{3cm}|p{7cm}|}
    \hline
    \textbf{Method}               & \textbf{Parameters}         & \textbf{Notes} \\ \hline
    Bistochastic L2               & $\varepsilon \in [10^{-2}, 10^2]$   & No hollowness constraint for ``fat'' variant \\ \hline
    Bistochastic Entropic         & $\varepsilon \in [10^{-2}, 10^2]$    & No hollowness constraint for ``fat'' variant \\ \hline
    $k$-NN         & $k \in [5, 125]$    & Symmetrised and symmetrically normalised \\ \hline
    Epanechnikov         & $h = \varepsilon^{2/3}$    & Bandwidth scaling consistent with $\varepsilon$, symmetrically normalised \\ \hline
    Gaussian & $h \in [10^{-2}, 10^2]$    & Symmetrically normalised \\ \hline
    Gaussian L2 & $h \in [10^{-2}, 10^2]$    & Bistochastic L2 projection of Gaussian kernel \citep{zass2006doubly} \\ \hline
    MAGIC \citep{van2018recovering} & $k \in [5, 125], t = 5$ & All other parameters are default \\ \hline
    \end{tabular}
    \caption{Parameter ranges chosen for each kernel construction method in the example of Figure \ref{fig:figure_spiral}.}
    \label{tab:spiral_methods_parameters}
  }
\end{table}

\vskip 0.2in
\bibliography{Refs}

\end{document}